\documentclass[12pt]{article}

\usepackage{amsmath}
\usepackage{amssymb}
\usepackage{booktabs}
\usepackage[dvipdfmx]{color}
\usepackage[dvipdfmx]{graphicx}

\usepackage{bm}
\usepackage{amsthm}
\usepackage{url}
\usepackage{subfigure}
\usepackage{algorithm}
\usepackage{algorithmic}
\usepackage{tabulary}
\usepackage[dvipdfmx]{graphicx}

\newtheorem{thm}{Theorem}[section]

\newtheorem{lem}{Lemma}[section]
\newtheorem{cor}{Corollary}[section]

\numberwithin{equation}{section}

\begin{document}
\makeatletter

\begin{center}
\large{\bf Theoretical analysis of Adam using hyperparameters close to one without Lipschitz smoothness}\\
\small{This work was supported by JSPS KAKENHI Grant Number 21K11773.}
\end{center}\vspace{3mm}

\begin{center}
\textsc{Hideaki Iiduka}\\
Department of Computer Science, 
Meiji University,
1-1-1 Higashimita, Tama-ku, Kawasaki-shi, Kanagawa 214-8571 Japan. 
(iiduka@cs.meiji.ac.jp)
\end{center}

\vspace{2mm}

\footnotesize{
\noindent\begin{minipage}{14cm}
{\bf Abstract:}
Convergence and convergence rate analyses of adaptive methods, such as Adaptive Moment Estimation (Adam) and its variants, have been widely studied for nonconvex optimization. The analyses are based on assumptions that the expected or empirical average loss function is Lipschitz smooth (i.e., its gradient is Lipschitz continuous) and the learning rates depend on the Lipschitz constant of the Lipschitz continuous gradient. Meanwhile, numerical evaluations of Adam and its variants have clarified that using small constant learning rates without depending on the Lipschitz constant and hyperparameters ($\beta_1$ and $\beta_2$) close to one is advantageous for training deep neural networks. Since computing the Lipschitz constant is NP-hard, the Lipschitz smoothness condition would be unrealistic. This paper provides theoretical analyses of Adam without assuming the Lipschitz smoothness condition in order to bridge the gap between theory and practice. The main contribution is to show theoretical evidence that Adam using small learning rates and hyperparameters close to one performs well, whereas the previous theoretical results were all for hyperparameters close to zero. Our analysis also leads to the finding that Adam performs well with large batch sizes. Moreover, we show that Adam performs well when it uses diminishing learning rates and hyperparameters close to one.

\end{minipage}
 \\[5mm]


\hbox to14cm{\hrulefill}\par


\section{Introduction}
\label{sec:1}
\subsection{Background}
\label{subsec:1.1}
One way to train a deep neural network is to find the model parameters of the network that minimize loss functions called the expected risk and empirical risk by using first-order optimization methods \cite[Section 4]{bottou}. The simplest optimizer is mini-batch stochastic gradient descent (SGD) \cite{robb1951,zinkevich2003,nem2009,gha2012,gha2013}. There are many deep learning optimizers to accelerate SGD, such as momentum methods \cite{polyak1964,nest1983} and adaptive methods, e.g., Adaptive Gradient (AdaGrad) \cite{adagrad}, Root Mean Square Propagation (RMSProp) \cite{rmsprop}, Adaptive Moment Estimation (Adam) \cite{adam}, Yogi \cite{NEURIPS2018_90365351}, Adaptive Mean Square Gradient (AMSGrad) \cite{reddi2018}, Adam with decoupled weight decay (AdamW) \cite{loshchilov2018decoupled}, and AdaBelief (named for adapting stepsizes by the belief in observed gradients) \cite{ada}.

Convergence and convergence rate analyses of deep learning optimizers have been widely studied for convex optimization \cite{zin2010,adam,reddi2018,luo2019,dun2020}. Meanwhile, theoretical investigations on deep learning optimizers for nonconvex optimization are needed so that these optimizers can be practically used for nonconvex optimization in deep learning \cite{kxu2015,ar2017,vas2017}.

Convergence analyses of SGD for nonconvex optimization were presented in \cite{feh2020,chen2020,sca2020,loizou2021} (see \cite{gower2021,loizou2021} for convergence analyses of SGD for two classes of nonconvex optimization problem, quasar-convex and Polyak--Lojasiewicz optimization). For example, Theorem 12 in \cite{sca2020} gave an upper bound of $(1/K) \sum_{k=1}^K \mathbb{E}[\|\nabla f (\bm{\theta}_k)\|^2]$ generated by SGD with a constant learning rate $\alpha = 1/L$ is $\mathcal{O}(1/\sqrt{K}) + C$, where $(\bm{\theta}_k)_{k\in\mathbb{N}}$ is the sequence generated by an optimizer, $L$ is the Lipschitz constant of the Lipschitz continuous gradient of the loss function $f \colon \mathbb{R}^d \to \mathbb{R}$, $K$ denotes the number of steps, and $C > 0$ is a constant. Convergence analyses depending on the batch size were presented in \cite{chen2020}. In particular, Theorem 3.2 in \cite{chen2020} indicates that running SGD with a diminishing learning rate $\alpha_k = 1/k$ and a large batch size for sufficiently many steps leads to convergence to a stationary point of the sum of loss functions.

\begin{table}[htbp]
\caption{Upper bounds of performance measures of optimizers with learning rate $\alpha_k$ and hyperparameters $\beta_1$ and $\beta_2$ for nonconvex optimization ($C, G > 0$, $s \in (0,1/2)$, $L$ denotes the Lipschitz constant of the Lipschitz continuous gradient of the loss function, $K$ denotes the number of steps, $b$ is the batch size, and $C_1$ is a monotone decreasing function. $\beta \approx a$ implies that, if $\beta$ is close to $a$, then the upper bounds are small.)}\label{table:1}
\centering 
\begin{tabular}{lccc}
\toprule
Optimizer & Learning Rate $\alpha_k$ & Parameters $\beta_1$, $\beta_2$ & Upper Bound \\
\midrule
SGD 
& $\displaystyle{\frac{1}{L}}$ 
& --- 
& $\displaystyle{\mathcal{O}\left(\frac{1}{K} \right) + C}$ \\
\cite{sca2020} & & & \\
\midrule
Adam 
& $\displaystyle{\mathcal{O}\left(\frac{1}{L} \right)}$ 
& $\displaystyle{\beta_2 \geq 1 - \mathcal{O}\left(\frac{1}{G^2}\right)}$ 
& $\displaystyle{\mathcal{O}\left(\frac{1}{K} + \frac{1}{b}\right)}$\\
\cite{NEURIPS2018_90365351} & & & \\
\midrule
Yogi 
& $\displaystyle{\mathcal{O}\left(\frac{1}{L} \right)}$ 
& $\displaystyle{\beta_2 \geq 1 - \mathcal{O}\left(\frac{1}{G^2}\right)}$ 
& $\displaystyle{\mathcal{O}\left(\frac{1}{K} + \frac{1}{b}\right)}$\\
\cite{NEURIPS2018_90365351} & & & \\
\midrule
Generic Adam 
& $\displaystyle{\mathcal{O}\left( \frac{1}{\sqrt{k}} \right)}$
& $\beta_1 \approx 0$, & $\displaystyle{\mathcal{O}\left(\frac{\log K}{\sqrt{K}}\right)}$ \\
\cite{cvpr2019} & & $\displaystyle{\beta_2 = 1 - \frac{1}{k} \approx 1}$ 
 & \\
\midrule
AdaFom 
& $\displaystyle{\frac{1}{\sqrt{k}}}$
& $\beta_1 \approx 0$
& $\displaystyle{\mathcal{O}\left(\frac{\log K}{\sqrt{K}}\right)}$ \\
\cite{chen2019} & & & \\
\midrule
AMSGrad 
& $\alpha$
& $0 \approx \beta_1 < \sqrt{\beta_2}$
& $\displaystyle{\mathcal{O}\left(\frac{1}{K^{\frac{1}{2} - s}}\right)}$ \\
\cite{opt2020} & & & \\
\midrule
AdaBelief 
& $\displaystyle{\mathcal{O}\left( \frac{1}{\sqrt{k}} \right)}$
& $\beta_1 \approx 0$, $\beta_2 \approx 0$
& $\displaystyle{\mathcal{O}\left(\frac{\log K}{\sqrt{K}}\right)}$ \\
\cite{ada} & & & \\
\midrule
Padam 
& $\alpha$
& $\beta_1 \approx 0$, $\beta_2 \approx 0$
& $\displaystyle{\mathcal{O}\left(\frac{1}{K^{\frac{1}{2} - s}}\right)}$ \\
\cite{chen2020_1} & & & \\
\midrule
Adaptive methods
& $\alpha$
& $\beta_1 \approx 0$, $\beta_2 \approx 0$
& $\displaystyle{\mathcal{O}\left(\frac{1}{K}\right) + C_1 (\alpha,\beta_1)}$ \\
\cite{iiduka2021} & & & \\
\midrule
Adaptive methods
& $\displaystyle{\frac{1}{\sqrt{k}}}$
& $\beta_1 \approx 0$, $\beta_2 \approx 0$
& $\displaystyle{\mathcal{O}\left(\frac{1}{\sqrt{K}}\right)}$ \\
\cite{iiduka2021} & & & \\
\bottomrule
\end{tabular}
\end{table}

Convergence analyses of adaptive methods for nonconvex optimization were presented in \cite{NEURIPS2018_90365351,cvpr2019,chen2019,opt2020,ada,chen2020_1,iiduka2021}. The previous results are summarized in Table \ref{table:1}. Theorems 1 and 2 in \cite{NEURIPS2018_90365351} indicate that, if $\alpha_k = \mathcal{O}(1/L)$ and a hyperparameter $\beta_2 \geq 1 - \mathcal{O}(1/G^2)$, then the upper bounds of $(1/K) \sum_{k=1}^K \mathbb{E}[\|\nabla f (\bm{\theta}_k)\|^2]$ generated by Adam and Yogi are each $\mathcal{O}(1/K + 1/b)$, where $G$ is the upper bound of the stochastic gradient and $b$ is the batch size. Theorem 2 in \cite{NEURIPS2018_d54e99a6} indicates that computing the Lipschitz constant $L$ is NP-hard. Hence, using a learning rate depending on the Lipschitz constant $L$ would be unrealistic. Convergence analyses of adaptive methods using diminishing learning rates that do not depend on $L$ were presented in \cite{cvpr2019,chen2019,ada,iiduka2021}, while convergence analyses of adaptive methods using constant learning rates that do not depend on $L$ were presented in \cite{opt2020,chen2020_1,iiduka2021}. These studies indicate that, if $K$ is sufficiently large and if $\beta_1$ and $\beta_2$ are close to $0$, then adaptive methods using learning rates that do not depend on  $L$ approximate the stationary points of $f$. For example, Theorem 3 in \cite{opt2020} indicates that AMSGrad using a constant learning rate $\alpha$ satisfies 
\begin{align*}
\frac{1}{K} \sum_{k=1}^K \mathbb{E}\left[ \|\nabla f (\bm{\theta}_k)\|^2 \right]
\leq 
\frac{M_1}{K \alpha} + \frac{M_2 d}{K} + \frac{\alpha M_3 d}{K^{\frac{1}{2}-s}},
\end{align*} 
where $s \in [0,1/2]$ and $M_i$ $(i=1,2,3)$ are positive constants, and $M_2 = \mathcal{O}(G^3/(1-\beta_1))$ and $M_3 = \mathcal{O}(G^2/(1-\beta_2))$ depend on $\beta_1$ and $\beta_2$ and the upper bound $G$. $M_2 $ and $M_3$ are small when $\beta_1$ and $\beta_2$ are small, i.e., $\beta_1, \beta_2 \approx 0$. Hence, using $\beta_1$ and $\beta_2$ close to $0$ is advantageous for AMSGrad.

Meanwhile, numerical evaluations have shown that using $\beta_1$ and $\beta_2$ such as 
\begin{align}\label{one}
\beta_1 \in \{ 0.9, 0.99\} \text{ and } \beta_2 \in \{ 0.99, 0.999 \}
\end{align}
is advantageous for training deep neural networks \cite{adam,reddi2018,NEURIPS2018_90365351,cvpr2019,chen2019,ada,chen2020_1}. The practically useful $\beta_1$ and $\beta_2$ defined by \eqref{one} are each close to $1$, whereas in contrast the theoretical results in Table \ref{table:1} show that using $\beta_1$ and $\beta_2$ close to $0$ makes the upper bounds of the performance measures small. Hence, there is a gap between theory ($\beta_1, \beta_2 \approx 0$; see also Table \ref{table:1}) and practice ($\beta_1, \beta_2 \approx 1$; see also \eqref{one}) for adaptive methods. In \cite{l.2018dont}, it was numerically shown that using an enormous batch size leads to a reduction in the number of parameter updates and model training time. The theoretical results in \cite{NEURIPS2018_90365351} showed that using large batch sizes makes the upper bound of $\mathbb{E}[\|\nabla f (\bm{\theta}_k)\|^2]$ small. Accordingly, the practical results for large batch sizes matches the theoretical ones. 

\subsection{Motivation}
As indicated in Section \ref{subsec:1.1}, it was numerically shown that the performance of an adaptive method strongly depends on the hyperparameters $\beta_1$ and $\beta_2$ being close to $1$. The motivation of this paper is to show {\em theoretically} evidence such that Adam performs well when $\beta_1$ and $\beta_2$ are each set close to $1$. Using the Lipschitz constant of the Lipschitz continuous gradient of the loss function would be unrealistic \cite{NEURIPS2018_d54e99a6}. Hence, it will not be assumed that the loss function $f$ is Lipschitz smooth (i.e., its gradient is Lipschitz continuous). The Lipschitz smoothness condition of $f$ implies that, for all $\bm{x},\bm{y} \in \mathbb{R}^d$, 
\begin{align}\label{smooth}
f(\bm{y}) \leq f(\bm{x}) + \nabla f(\bm{x})^\top (\bm{y} - \bm{x}) + \frac{L}{2}
\|\bm{x} - \bm{y}\|^2.
\end{align}
Almost all of the previous analyses of adaptive methods are based on the descent lemma \eqref{smooth}, and hence, they can use the expectation of the squared norm of the full gradient $\mathbb{E}[\|\nabla f (\bm{\theta}_k)\|^2]$ as the performance measure. Since we do not assume Lipschitz smoothness of the loss function, we cannot use \eqref{smooth}. Accordingly, we must use other performance measures that are different from $\mathbb{E}[\|\nabla f (\bm{\theta}_k)\|^2]$.

\subsubsection{Performance measure}
This paper considers the Adam optimizer \cite{adam}, which is defined for all $k\in \mathbb{N}$ by 
\begin{align}\label{adam}
\bm{\theta}_{k+1} 
:= \bm{\theta}_k - \frac{\alpha_k}{1 - \beta_{1}^{k+1}} 
\mathsf{diag}\left(\hat{v}_{k,i}^{- \frac{1}{2}} \right) {\bm{m}}_k,
\end{align}
where $\alpha_k > 0$ is the learning rate, $\beta_i \in (0,1)$ $(i=1,2)$, $\hat{v}_{k,i} := (1 - \beta_2^{k+1})^{-1} v_{k,i}$, $v_{k,i} := \beta_2 v_{k-1,i} + (1-\beta_2) g_{k,i}^2$, $m_{k,i} := \beta_1 m_{k-1,i} + (1 - \beta_1) g_{k,i}$, $\bm{m}_k := (m_{k,i})_{i=1}^d$, $\bm{g}_k := (g_{k,i})_{i=1}^d$ is the stochastic gradient, and $\mathsf{diag}(x_i)$ is a diagonal matrix with diagonal components $x_1, x_2, \ldots, x_d$ (see also Algorithm \ref{algo:1} for details). We use the following theoretical performance measure to approximate a local minimizer $\bm{\theta}^\star$ of the nonconvex optimization problem of minimizing the loss function $f \colon \mathbb{R}^d \to \mathbb{R}$ over $\mathbb{R}^d$:
\begin{align}\label{pm_1}
\mathbb{E}\left[ \bm{m}_k^\top (\bm{\theta}_k - \bm{\theta}^\star) 
\right] 
\leq \epsilon,
\end{align}
where $\epsilon > 0$ is the precision. The performance measure \eqref{pm_1} is an $\epsilon$-approximation of the sequence $(\bm{\theta}_k)_{k\in\mathbb{N}}$ generated by Adam \eqref{adam} in the sense that the inner product of the vector $\bm{\theta}_k - \bm{\theta}^\star$ and the inverse direction $\bm{m}_k$ of the search direction $-\bm{m}_k$ is less than or equal to $\epsilon$. Thanks to the previous theoretical results shown in Table \ref{table:1}, it is guaranteed that Adam and its variants can find a stationary point $\bm{\theta}^\star$ of $f$, which implies that, for a sufficiently large step size $k$, $0\leq \mathbb{E}[\bm{m}_k^\top (\bm{\theta}_k - \bm{\theta}^\star)]$. Therefore, it is sufficient to check whether or not Adam satisfies \eqref{pm_1}. If $\mathbb{E}[ \bm{m}_k^\top (\bm{\theta}_k - \bm{\theta}^\star)] \approx 0$ for a sufficiently large $k$, then $\mathbb{E}[\|\bm{m}_k\|]$ or $\mathbb{E}[\|\bm{\theta}_k - \bm{\theta}^\star\|]$ will be approximately zero. We also use the mean value of $\mathbb{E}[\bm{m}_k^\top (\bm{\theta}_k - \bm{\theta}^\star)]$ ($k\in \{1,2,\ldots,K\}$), that is, 
\begin{align}\label{pm_2}
\frac{1}{K} \sum_{k=1}^K \mathbb{E}\left[\bm{m}_k^\top (\bm{\theta}_k - \bm{\theta}^\star) \right] \leq \epsilon.
\end{align}
Let $\bm{\theta}^\star \in \mathbb{R}^d$. $\bm{\theta}^\star$ is a stationary point of $f$ if and only if $\nabla f (\bm{\theta}^\star)^\top (\bm{\theta}^\star - \bm{\theta}) \leq 0$ for all $\bm{\theta} \in \mathbb{R}^d$. Hence, we also use the following performance measures: for all $\bm{\theta} \in \mathbb{R}^d$, 
\begin{align}\label{pm_3}
\mathbb{E}\left[
\nabla f (\bm{\theta}_k)^\top (\bm{\theta}_k - \bm{\theta}) 
\right] \leq \epsilon
\end{align}
and 
\begin{align}\label{pm_4}
\frac{1}{K} \sum_{k=1}^K \mathbb{E}\left[
\nabla f (\bm{\theta}_k)^\top (\bm{\theta}_k - \bm{\theta}) 
\right] \leq \epsilon.
\end{align}
The advantage of using \eqref{pm_1}, \eqref{pm_2}, \eqref{pm_3}, and \eqref{pm_4} is that we can evaluate the upper bound of the performance measure for Adam without assuming that the loss function $f$ is Lipschitz smooth.

\subsection{Our results and contribution}
Numerical evaluations presented in \cite{adam,reddi2018,NEURIPS2018_90365351,cvpr2019,chen2019,ada,chen2020_1} showed that Adam and its variants perform well when they use a small constant learning rate $\alpha$ and hyperparameters $\beta_1$ and $\beta_2$ with values close to $1$. Hence, we would like to show theoretical evidence that Adam performs well when we set a small $\alpha$ and $\beta_1$ and $\beta_2$ close to $1$. In particular, we will show that Adam \eqref{adam} using $\alpha > 0$, $\beta_1 \in (0,1)$, and $\beta_2 \in [0,1)$ satisfies 
\begin{align}\label{upper}
\begin{split}
\mathbb{E}\left[
\bm{m}_{k}^\top (\bm{\theta}_k - \bm{\theta}^\star)
\right]
&\leq
\underbrace{\frac{D(\bm{\theta}^\star) M^{\frac{1}{4}}}{v_*^{\frac{1}{4}} \beta_1}\sqrt{ \frac{\sigma^2}{b} + G^2}}_{C_1(\beta_1,b)}
+
\underbrace{\frac{\alpha \sqrt{1-\beta_2^{k+1}}}{2 \sqrt{v_*} \beta_{1} (1-\beta_1^{k+1})} \left( \frac{\sigma^2}{b} + G^2 \right)}_{C_2(\alpha,\beta_1,\beta_2,b,k)}\\
&\quad 
+ \underbrace{\frac{1 - \beta_1}{\beta_1} D(\bm{\theta}^\star) G}_{C_3(\beta_1)}
+ \underbrace{(1-\beta_1) D(\bm{\theta}^\star)
\left(B + \sqrt{ \frac{\sigma^2}{b} + G^2} \right)}_{C_4(\beta_1,b)}
\end{split}
\end{align}
and
\begin{align*}
\mathbb{E}\left[
\nabla f(\bm{\theta}_k)^\top (\bm{\theta}_k - \bm{\theta})
\right]
&\leq
\underbrace{\frac{D(\bm{\theta}) M^{\frac{1}{4}}}{v_*^{\frac{1}{4}} \beta_1}\sqrt{ \frac{\sigma^2}{b} + G^2}}_{C_1(\beta_1,b)}
+
\underbrace{\frac{\alpha \sqrt{1-\beta_2^{k+1}}}{2 \sqrt{v_*} \beta_{1} (1-\beta_1^{k+1})} \left( \frac{\sigma^2}{b} + G^2 \right)}_{C_2(\alpha,\beta_1,\beta_2,b,k)}\\
&\quad +
\underbrace{\frac{1 - \beta_{1}}{\beta_{1}}D(\bm{\theta}) G }_{C_3(\beta_1)}
+ 
\underbrace{\left(\frac{1}{\beta_1} + 2 (1 - \beta_{1}) \right)
D(\bm{\theta}) \left(B + \sqrt{\frac{\sigma^2}{b} + G^2} \right)}_{C_5(\beta_1,b)},
\end{align*} 
where $\bm{\theta}^\star$ is a stationary point of $f$ and $\bm{\theta} \in \mathbb{R}^d$ (see Theorem \ref{c1} for the definitions of the parameters). $C_1, C_3, C_4$, and $C_5$ are monotone decreasing for $\beta_1 \in (0,1)$ and $b > 0$. Hence, it is desirable to set $\beta_1$ close to $1$ such as $\beta_1 = 0.9$ and a large batch size $b$. Although a function $f_k(\beta_1) := 1/(\beta_1(1-\beta_1^{k+1}))$ in $C_2$ is monotone increasing for $\beta_1$ satisfying $\beta_1^{k+1} > 1/(k+2)$, $f_k(\beta_1)$ with a sufficiently large $k$ is monotone decreasing for $\beta_1$. $f_{\beta_1}(k) := 1/(\beta_1(1-\beta_1^{k+1}))$ is monotone decreasing for $k \in \mathbb{N}$. Although $f_{\beta_2}(k) := (1-\beta_2^{k+1})^{1/2}$ is monotone increasing for $k \in \mathbb{N}$, $f_{\beta_2}(k) := (1-\beta_2^{k+1})^{1/2} \leq 1$ is small for all $\beta_2$ and all $k \in\mathbb{N}$. $C_2$ is monotone increasing for $\alpha$ and monotone decreasing for $\beta_2$. Hence, we must set a small $\alpha$ and $\beta_2$ close to $1$, such as $\alpha = 10^{-3}$ and $\beta_2 = 0.99$, to make $C_2$ small. For simplicity, we may evaluate 
\begin{align*}
C_2(\alpha,\beta_1,\beta_2,b,k)
:= \frac{\alpha \sqrt{1-\beta_2^{k+1}}}{2 \sqrt{v_*} \beta_{1} (1-\beta_1^{k+1})} \left( \frac{\sigma^2}{b} + G^2 \right)
\leq 
\frac{\alpha}{2 \sqrt{v_*} \beta_{1} (1-\beta_1)} \left( \frac{\sigma^2}{b} + G^2 \right).
\end{align*}
Since $1/(\beta_1(1-\beta_1))$ is monotone increasing for $\beta_1 \geq 1/2$, we need to set a small learning rate $\alpha$ to make $\alpha/(\beta_1(1-\beta_1))$ small. Therefore, using a small learning rate $\alpha$ and $\beta_1$ and $\beta_2$ close to $1$ is advantageous for Adam (see Theorem \ref{d1} for the results for when Adam uses a diminishing learning rate $\alpha_k$). Moreover, \eqref{upper} implies that 
\begin{align*}
\limsup_{k \to +\infty} \mathbb{E}\left[
\bm{m}_{k}^\top (\bm{\theta}_k - \bm{\theta}^\star)
\right]
&\leq
\underbrace{\frac{D(\bm{\theta}^\star) M^{\frac{1}{4}}}{v_*^{\frac{1}{4}} \beta_1}\sqrt{ \frac{\sigma^2}{b} + G^2}}_{C_1(\beta_1,b)}
+
\underbrace{\frac{\alpha}{2 \sqrt{v_*} \beta_{1}} \left( \frac{\sigma^2}{b} + G^2 \right)}_{\bar{C}_2(\alpha,\beta_1,b)}\\
&\quad 
+ \underbrace{\frac{1 - \beta_1}{\beta_1} D(\bm{\theta}^\star) G}_{C_3(\beta_1)}
+ \underbrace{(1-\beta_1) D(\bm{\theta}^\star)
\left(B + \sqrt{ \frac{\sigma^2}{b} + G^2} \right)}_{C_4(\beta_1,b)},
\end{align*}
which implies that we must set a small learning rate $\alpha$, $\beta_1$ close to $1$, and a large batch size $b$ for $C_1$, $\bar{C}_2$, $C_3$, and $C_4$ to be small (see also Corollary \ref{cor:1}).

A stochastic convex optimization problem exists such that Adam using $\beta_1 < \sqrt{\beta_2}$ (e.g., $\beta_1 = 0.9$ and $\beta_2 = 0.999$) does not converge to the optimal solution \cite[Theorem 3]{reddi2018}, so it is not guaranteed that Adam can solve every nonconvex optimization problem. Reddi et al showed that, if $v_{k,i}$ in \eqref{adam} satisfies 
\begin{align}\label{max_1}
\hat{v}_{k+1,i} \geq \hat{v}_{k,i} \text{ for all } k \in \mathbb{N} 
\text{ and all } i \in [d],
\end{align}
then Adam (AMSGrad) with a diminishing learning rate $\alpha_k$ and $\beta_1$ and $\beta_2$ defined by
\begin{align}\label{dim_1}
\alpha_k = \mathcal{O}\left(\frac{1}{\sqrt{k}} \right)
\text{ and }
\beta_1 < \sqrt{\beta_2}
\end{align}
can solve convex optimization problems \cite[(2), Algorithm 2, Theorem 4]{reddi2018} (see also \cite[Theorems 2.1 and 2.2]{ada} for the convergence analyses of AdaBelief using \eqref{max_1}). Motivated by the results in \cite{reddi2018,ada}, we decided to study Adam under Condition \eqref{max_1}. Our results are summarized in Table \ref{table:2} (see Theorems \ref{c2}, \ref{c3}, \ref{d2}, and \ref{d3} for the details). While the previous results shown in Table \ref{table:1} used $\beta_1, \beta_2 \approx 0$, our results use $\beta_1, \beta_2 \approx 1$ to make the upper bound of \eqref{pm_2} small (see Theorems \ref{c2}, \ref{c3}, \ref{d2}, and \ref{d3} for the upper bounds of \eqref{pm_4}). Therefore, the results we present in this paper are theoretical confirmation of the numerical evaluations \cite{adam,reddi2018,NEURIPS2018_90365351,cvpr2019,chen2019,ada,chen2020_1} showing that Adam and its variants using $\beta_1$ and $\beta_2$ close to $1$ perform well.

\begin{table}[htbp]
\caption{Upper bounds of the performance measure \eqref{pm_2} of Adam with learning rate $\alpha_k$ and hyperparameters $\beta_{1}$ and $\beta_{2k}$ for nonconvex optimization ($a > 0$, $C_i$ $(i=1,2,3)$ are constants, $K$ is the number of steps, $b$ is the batch size, and $n$ is the number of samples. $\beta \approx \gamma$ implies that, if $\beta$ is close to $\gamma$, then the upper bounds are small.)}\label{table:2}
\centering 
\begin{tabular}{lcccc}
\toprule
Optimizer & Learning Rate $\alpha_k$ & Parameters $\beta_1$, $\beta_{2k}$ & Batch $b$ & Upper Bound of \eqref{pm_2}\\
\midrule
Adam with \eqref{max_1} 
& $\alpha$ 
& $\beta_1 \approx 1$ 
& $b \approx n$
& $\displaystyle{\mathcal{O}\left(\frac{1}{K} \right) + C_1 \alpha}$ \\
(Theorem \ref{c2}) & & $\beta_2 \approx 1$ & & $\displaystyle{+ C_2 \frac{\alpha}{b} + C_3 \frac{1 - \beta_1}{\beta_1}}$ \\
\midrule
Adam with \eqref{max_1} 
& $\alpha$ 
& $\beta_1 \approx 1$ 
& $b \approx n$
& $\displaystyle{\mathcal{O}\left(\frac{1}{\sqrt{K}} + \frac{1 - \beta_1}{\beta_1} \right)}$ \\
(Theorem \ref{c3}) & & $\beta_{2k} \to 1$ & & \\
\midrule
Adam with \eqref{max_1} 
& $\displaystyle{\frac{1}{\sqrt{k}}}$ 
& $\beta_1 \approx 1$
& $b \approx n$ 
& $\displaystyle{\mathcal{O}\left(\frac{1}{\sqrt{K}} + \frac{1 - \beta_1}{\beta_1} \right)}$ \\
(Theorem \ref{d2}) & & $\beta_2 \approx 1$ & & \\
\midrule
Adam with \eqref{max_1} 
& $\displaystyle{\frac{1}{k^a}}$ 
& $\beta_1 \approx 1$ 
& $b \approx n$
& $\displaystyle{\mathcal{O}\left(\frac{1}{\sqrt{K}} + \frac{1 - \beta_1}{\beta_1} \right)}$ \\
(Theorem \ref{d3}) & & $\beta_{2k} \to 1$ & & \\
\bottomrule
\end{tabular}
\end{table}

Our first contribution is to provide theoretical analyses of Adam without assuming the Lipschitz smoothness condition. Since we cannot use the descent lemma \eqref{smooth} based on this condition, we cannot use the performance measure $\mathbb{E}[\|\nabla f(\bm{\theta}_k)\|^2]$. Instead, we use two performance measures \eqref{pm_1} and \eqref{pm_3}: \eqref{pm_1} is the inner product of $\bm{\theta}_k - \bm{\theta}^\star$ and the inverse direction $\bm{m}_k$ of the search direction generated by Adam. If the \eqref{pm_1} is approximately zero, then Adam can approximate a local minimizer $\bm{\theta}^\star$ of the problem of minimizing $f$. \eqref{pm_3} is the inner product of $\bm{\theta}_k - \bm{\theta}$ and the full gradient $\nabla f(\bm{\theta}_k)$ generated by Adam. If \eqref{pm_3} is approximately zero, then Adam can approximate a local minimizer $\bm{\theta}^\star$ of the problem of minimizing $f$. \eqref{pm_1} and \eqref{pm_3} and the mean values \eqref{pm_2} and \eqref{pm_4} of \eqref{pm_1} and \eqref{pm_3} can be used to evaluate the performance of deep learning optimizers without assuming the Lipschitz smoothness condition.

While the numerical evaluations presented in \cite{adam,reddi2018,NEURIPS2018_90365351,cvpr2019,chen2019,ada,chen2020_1} have shown that adaptive methods using $\beta_1$ and $\beta_2$ close to $1$ are advantageous for training deep neural networks, the theoretical results in Table \ref{table:1} imply that adaptive methods with $\beta_1$ and $\beta_2$ close to $0$ are good for solving nonconvex optimization problems in deep learning. This implies that there is a large gap between theory and practice for adaptive methods. Our results in Table \ref{table:2} show that Adam indeed performs well when $\beta_1$ and $\beta_2$ are set close to $1$. Thus, the gap between theory and practice can be bridged for Adam. Our results also show that using a large batch size makes the upper bounds of the performance measures small, which implies that our results match the numerical evaluations in \cite{l.2018dont}.

The remainder of the paper is as follows. First, the mathematical preliminaries are laid out in Section \ref{sec:2}, with the definitions of nonconvex optimization and the Adam optimizer. Section \ref{sec:3} describes the theoretical results in detail. Finally, a brief summary and outline of future work are presented in Section \ref{sec:4}.

\section{Mathematical preliminaries}
\label{sec:2}
\subsection{Nonconvex optimization}
Let $\mathbb{R}^d$ be a $d$-dimensional Euclidean space with inner product $\langle \bm{x},\bm{y} \rangle := \bm{x}^\top \bm{y}$ inducing the norm $\| \bm{x}\|$ and $\mathbb{N}$ the set of nonnegative integers. Let $[d] := \{1,2,\ldots,d\}$ for $d \geq 1$. The mathematical model used in this paper is based on \cite{shallue2019}. Given a parameter $\bm{\theta} \in \mathbb{R}^d$ and given a data point $z$ in a data domain $Z$, a machine learning model provides a prediction whose quality is measured by a differentiable nonconvex loss function $\ell(\bm{\theta};z)$. We aim to minimize the expected loss defined for all $\bm{\theta} \in \mathbb{R}^d$ by
\begin{align}\label{expected}
f(\bm{\theta}) = \mathbb{E}_{z \sim \mathcal{D}} 
[\ell(\bm{\theta};z) ]
= \mathbb{E}[ \ell_{\xi} (\bm{\theta}) ],
\end{align}
where $\mathcal{D}$ is a probability distribution over $Z$, $\xi$ denotes a random variable with distribution function $P$, and $\mathbb{E}[\cdot]$ denotes the expectation taken with respect to $\xi$. A particularly interesting example of \eqref{expected} is the empirical average loss defined for all $\bm{\theta} \in \mathbb{R}^d$ by 
\begin{align}\label{empirical}
f(\bm{\theta}; S) = \frac{1}{n} \sum_{i\in [n]} \ell(\bm{\theta};z_i)
= \frac{1}{n} \sum_{i\in [n]} \ell_i(\bm{\theta}),
\end{align}
where $S = (z_1, z_2, \ldots, z_n)$ denotes the training set, $\ell_i (\cdot) := \ell(\cdot;z_i)$ denotes the loss function corresponding to the $i$-th training data $z_i$, and $[n] := \{1,2,\ldots,n\}$. Our main objective is to find a local minimizer of $f$ over $\mathbb{R}^d$, i.e., a stationary point $\bm{\theta}^\star \in \mathbb{R}^d$ satisfying $\nabla f(\bm{\theta}^\star) = \bm{0}$.

\subsection{Adam}
We assume that a stochastic first-order oracle (SFO) exists such that, for a given $\bm{\theta} \in \mathbb{R}^d$, it returns a stochastic gradient $\mathsf{G}_{\xi}(\bm{\theta})$ of the function $f$ defined by \eqref{expected}, where a random variable $\xi$ is supported on $\Xi$ independently of $\bm{\theta}$. Throughout this paper, we will assume the following standard conditions:
\begin{enumerate}
\item[(S1)] $f \colon \mathbb{R}^d \to \mathbb{R}$ defined by \eqref{expected} is continuously differentiable;
\item[(S2)] Let $(\bm{\theta}_k)_{k\in \mathbb{N}} \subset \mathbb{R}^d$ be the sequence generated by a deep learning optimizer. For each iteration $k$, 
\begin{align}\label{gradient}
\mathbb{E}_{\xi_k}[ \mathsf{G}_{\xi_k}(\bm{\theta}_k)] = \nabla f(\bm{\theta}_k),
\end{align}
where $\xi_0, \xi_1, \ldots$ are independent samples and the random variable $\xi_k$ is independent of $(\bm{\theta}_l)_{l=0}^k$. There exists a nonnegative constant $\sigma^2$ such that 
\begin{align}\label{sigma}
\mathbb{E}_{\xi_k}\left[ \left\|\mathsf{G}_{\xi_k}(\bm{\theta}_k) - 
\nabla f(\bm{\theta}_k) \right\|^2 \right] \leq \sigma^2.
\end{align}
\item[(S3)] For each iteration $k$, the optimizer samples a batch $B_{k}$ of size $b$ independently of $k$ and estimates the full gradient $\nabla f$ as 
\begin{align*}
\nabla f_{B_k} (\bm{\theta}_k)
:= \frac{1}{b} \sum_{i\in [b]} \mathsf{G}_{\xi_{k,i}}(\bm{\theta}_k),
\end{align*}
where $\xi_{k,i}$ is a random variable generated by the $i$-th sampling in the $k$-th iteration. 
\end{enumerate}
In the case that $f$ is defined by \eqref{empirical}, we have that, for each $k$, $B_k \subset [n]$ and 
\begin{align*}
\nabla f_{B_k} (\bm{\theta}_k)
= \frac{1}{b} \sum_{i \in [b]} \nabla \ell_{\xi_{k,i}} (\bm{\theta}_k).
\end{align*}

Algorithm \ref{algo:1} is the Adam optimizer under (S1)--(S3). The symbol $\odot$ in step 6 is defined for all $\bm{x} = (x_i)_{i=1}^d \in \mathbb{R}^d$, $\bm{x} \odot \bm{x} := (x_i^2)_{i=1}^d \in \mathbb{R}^d$. $\mathsf{diag}(x_i)$ in step 8 is a diagonal matrix with diagonal components $x_1, x_2, \ldots, x_d$.

\begin{algorithm} 
\caption{Adam \cite{adam}} 
\label{algo:1} 
\begin{algorithmic}[1] 
\REQUIRE
$\alpha_k \in (0,+\infty)$, 
$b \in (0,+\infty)$,
$\beta_{1k} \in (0,1)$, 
$\beta_{2k} \in [0,1)$
\STATE
$k \gets 0$, $\bm{\theta}_{0} \in\mathbb{R}^d$, $\bm{m}_{-1} := \bm{0}$, 
$\bm{v}_{-1} := \bm{0}$
\LOOP 
\STATE
$\nabla f_{B_k} (\bm{\theta}_k)
:= b^{-1} \sum_{i\in [b]} \mathsf{G}_{\xi_{k,i}}(\bm{\theta}_k)$
\STATE 
$\bm{m}_k := \beta_{1k} \bm{m}_{k-1} + (1-\beta_{1k}) \nabla f_{B_k}(\bm{\theta}_k)$
\STATE
$\displaystyle{\hat{\bm{m}}_k := (1-\beta_{1k}^{k+1})^{-1}\bm{m}_k}$
\STATE
$\bm{v}_k := \beta_{2k} \bm{v}_{k-1} + (1-\beta_{2k}) \nabla f_{B_k}(\bm{\theta}_k) \odot \nabla f_{B_k}(\bm{\theta}_k)$
\STATE
$\displaystyle{\hat{\bm{v}}_k := (1-\beta_{2k}^{k+1})^{-1}\bm{v}_k}$
\STATE
$\mathsf{H}_k := \mathsf{diag}(\sqrt{\hat{v}_{k,i}})$
\STATE 
$\bm{\theta}_{k+1} := \bm{\theta}_k - \alpha_k \mathsf{H}_k^{-1} \hat{\bm{m}}_k$
\STATE $k \gets k+1$
\ENDLOOP 
\end{algorithmic}
\end{algorithm}

\section{Analysis of Adam for Nonconvex Optimization}
\label{sec:3}
This section provides our detailed theoretical results for Adam. 

\subsection{Theoretical advantage of setting a small constant learning rate $\alpha$ and hyperparameters $\beta_1$ and $\beta_2$ close to $1$}
\label{subsec:3.1}

\subsubsection{Constant learning rate rule}
Let us consider Adam defined by Algorithm \ref{algo:1} using the following constant learning rate and hyperparameters:
\begin{align}\label{constant_lr}
\alpha_{k} := \alpha \in (0,+\infty), \text{ }
\beta_{1k} := \beta_1 \in (0,1), \text{ and }
\beta_{2k} := \beta_2 \in [0,1).
\end{align}
We assume the following conditions that were used in \cite[Theorem 4.1]{adam}:
\begin{enumerate}
\item[(A1)] There exist positive numbers $G$ and $B$ such that, for all $k\in \mathbb{N}$, $\| \nabla f (\bm{\theta}_k) \| \leq G$ and $\|\nabla f_{B_k}(\bm{\theta}_k)\| \leq B$.
\item[(A2)] For all $\bm{\theta} \in \mathbb{R}^d$, there exists a positive number $D(\bm{\theta})$ such that, for all $k\in \mathbb{N}$, $\| \bm{\theta}_k - \bm{\theta} \| \leq D(\bm{\theta})$.
\end{enumerate}

The following is an analysis of Adam using a constant learning rate and hyperparameters (see Appendix \ref{app:1} for the proof of Theorem \ref{c1}).

\begin{thm}\label{c1}
Suppose that (S1)--(S3) and (A1)--(A2) hold and $\bm{\theta}^\star \in \mathbb{R}^d$ is a stationary point of $f$. Then, Adam defined by Algorithm \ref{algo:1} using \eqref{constant_lr} satisfies that, for all $k\in \mathbb{N}$, 
\begin{align*}
\mathbb{E}\left[
\bm{m}_{k}^\top (\bm{\theta}_k - \bm{\theta}^\star)
\right]
&\leq
\underbrace{\frac{D(\bm{\theta}^\star) M^{\frac{1}{4}}}{v_*^{\frac{1}{4}} \beta_1}\sqrt{ \frac{\sigma^2}{b} + G^2}}_{C_1(\beta_1,b)}
+
\underbrace{\frac{\alpha \sqrt{1-\beta_2^{k+1}}}{2 \sqrt{v_*} \beta_{1} (1-\beta_1^{k+1})} \left( \frac{\sigma^2}{b} + G^2 \right)}_{C_2(\alpha,\beta_1,\beta_2,b,k)}\\
&\quad + \underbrace{\frac{1 - \beta_1}{\beta_1} D(\bm{\theta}^\star) G}_{C_3(\beta_1)} 
+ \underbrace{(1-\beta_1) D(\bm{\theta}^\star) \left( B + \sqrt{ \frac{\sigma^2}{b} + G^2} \right)}_{C_4(\beta_1,b)}
\end{align*}
and for all $\bm{\theta} \in \mathbb{R}^d$ and all $k\in \mathbb{N}$,
\begin{align*}
\mathbb{E}\left[
\nabla f(\bm{\theta}_k)^\top (\bm{\theta}_k - \bm{\theta})
\right]
&\leq
\underbrace{\frac{D(\bm{\theta}) M^{\frac{1}{4}}}{v_*^{\frac{1}{4}} \beta_1}\sqrt{ \frac{\sigma^2}{b} + G^2}}_{C_1(\beta_1,b)}
+
\underbrace{\frac{\alpha \sqrt{1-\beta_2^{k+1}}}{2 \sqrt{v_*} \beta_{1} (1-\beta_1^{k+1})} \left( \frac{\sigma^2}{b} + G^2 \right)}_{C_2(\alpha,\beta_1,\beta_2,b,k)}\\
&\quad +
\underbrace{\frac{1 - \beta_{1}}{\beta_{1}}D(\bm{\theta}) G}_{C_3(\beta_1)}
+ 
\underbrace{\left(\frac{1}{\beta_1} + 2 (1 - \beta_{1}) \right)
D(\bm{\theta}) \left(B + \sqrt{\frac{\sigma^2}{b} + G^2} \right)}_{C_5(\beta_1,b)},
\end{align*}
where $D(\bm{\theta})$, $G$, and $B$ are defined as in (A1) and (A2), $\nabla f_{B_k}(\bm{\theta}_k) \odot \nabla f_{B_k}(\bm{\theta}_k) := (g_{k,i}^2) \in \mathbb{R}_{+}^d$, $M := \sup\{\max_{i\in [d]} g_{k,i}^2 \colon k\in \mathbb{N}\} < + \infty$, and ${v}_* := \inf \{ \min_{i\in [d]} {v}_{k,i} \colon k\in \mathbb{N}\}$.
\end{thm}

We would like to set constant parameters $\alpha$, $\beta_1$, and $\beta_2$ so that the upper bounds of the performance measures $\mathbb{E}[\bm{m}_k^\top (\bm{\theta}_k - \bm{\theta}^\star)]$ and $\mathbb{E}[\nabla f(\bm{\theta}_k)^\top (\bm{\theta}_k - \bm{\theta})]$, denoted by $C_i$ ($i=1,2,3,4,5$) in Theorem \ref{c1}, can be small. Let $\bm{\theta} \in \mathbb{R}^d$ be fixed. We can check that 
\begin{align*}
C_3 (\beta_1) := \frac{1 - \beta_{1}}{\beta_{1}} D(\bm{\theta}) G
\end{align*} 
is monotone decreasing for $\beta_1 \in (0,1)$. Moreover, 
\begin{align*}
&C_1 (\beta_1, b) := \frac{D(\bm{\theta}) M^{\frac{1}{4}}}{{{v}_*^{\frac{1}{4}}} \beta_{1}}
\sqrt{ \frac{\sigma^2}{b} + G^2},
\text{ }
C_4 (\beta_1, b) := (1 - \beta_{1}) D(\bm{\theta})
\left(B + \sqrt{\frac{\sigma^2}{b} + G^2} \right),\\
&C_5 (\beta_1, b) := \left(\frac{1}{\beta_1} + 2 (1 - \beta_{1}) \right)
D(\bm{\theta})
\left(B + \sqrt{\frac{\sigma^2}{b} + G^2} \right)
\end{align*} 
are monotone decreasing for $\beta_1 \in (0,1)$ and $b > 0$. Therefore, we should set $\beta_1$ close to $1$ such as $\beta_1 = 0.9$ \cite{adam}. Let us consider 
\begin{align}\label{tilde_c_2}
C_2 (\alpha,\beta_1,\beta_2,b,k) := \frac{\alpha \sqrt{1-\beta_2^{k+1}}}{2 \sqrt{v_*} \beta_{1} (1-\beta_1^{k+1})} \left( \frac{\sigma^2}{b} + G^2 \right)
\leq
\frac{\alpha }{2 \sqrt{v_*} \beta_{1} (1-\beta_1)} \left( \frac{\sigma^2}{b} + G^2 \right)
=: U_2.
\end{align}
$C_2$ is monotone decreasing for $\beta_2 \in [0,1)$ and $b > 0$. Although the function $f_k(\beta_1) := 1/(\beta_1(1-\beta_1^{k+1}))$ in $C_2$ is monotone increasing for $\beta_1$ satisfying $\beta_1^{k+1} > 1/(k+2)$, $f_k(\beta_1)$ with a sufficiently large number $k$ is monotone decreasing for $\beta_1$. Moreover, $f_{\beta_1}(k) := 1/(\beta_1(1-\beta_1^{k+1}))$ is monotone decreasing for $k \in \mathbb{N}$. Although a function $f_{\beta_2}(k) := (1-\beta_2^{k+1})^{1/2}$ is monotone increasing for $k \in \mathbb{N}$, $f_{\beta_2}(k) := (1-\beta_2^{k+1})^{1/2} \leq 1$ is small for all $\beta_2$ and all $k \in\mathbb{N}$. $C_2$ is monotone increasing for $\alpha$ and monotone decreasing for $\beta_2$. Hence, we must set a small learning rate $\alpha$ and $\beta_2$ sufficiently close to $1$, for example, $\alpha = 10^{-3}$ and $\beta_2 = 0.999$ \cite{adam}, so that $C_2$ will be small when $\beta_1$ is close to $1$. In a simplistic form, the upper bound $U_2$ of $C_2$ depends on $\alpha/(\beta_1 (1-\beta_1))$, which is monotone increasing for $\beta_1 \geq 1/2$. Hence, using $\beta_1$ close to $1$ (e.g., $\beta_1 = 0.9$) implies that $U_2$ is large. Accordingly, we must set a small learning rate $\alpha$ so that $U_2$ will be small when $\beta_1$ is close to $1$.

Therefore, Theorem \ref{c1} indicates that using a small $\alpha$ and $\beta_1$ and $\beta_2$ close to $1$ is useful to implement Adam defined by Algorithm \ref{algo:1}, as shown by the previous numerical results \cite{adam,zhang2019}. Moreover, $C_1$, $C_2$, $C_4$, and $C_5$ are monotone decreasing with the batch size $b$. Hence, we should use a large batch size $b$ to implement Adam defined by Algorithm \ref{algo:1}, as shown by the previous numerical results \cite{l.2018dont}.

Theorem \ref{c1} leads to the following corollary: 

\begin{cor}\label{cor:1}
Under the assumptions in Theorem \ref{c1}, Adam defined by Algorithm \ref{algo:1} with \eqref{constant_lr} satisfies 
\begin{align*}
\limsup_{k \to + \infty} \mathbb{E}\left[
\bm{m}_{k}^\top (\bm{\theta}_k - \bm{\theta}^\star)
\right]
&\leq
\underbrace{\frac{D(\bm{\theta}^\star) M^{\frac{1}{4}}}{v_*^{\frac{1}{4}} \beta_1}\sqrt{ \frac{\sigma^2}{b} + G^2}}_{C_1(\beta_1,b)}
+
\underbrace{\frac{\alpha}{2 \sqrt{v_*} \beta_{1}} \left( \frac{\sigma^2}{b} + G^2 \right)}_{\bar{C}_2(\alpha,\beta_1,b)}\\
&\quad + \underbrace{\frac{1 - \beta_1}{\beta_1} D(\bm{\theta}^\star) G}_{C_3(\beta_1)} 
+ \underbrace{(1-\beta_1) D(\bm{\theta}^\star) \left( B + \sqrt{ \frac{\sigma^2}{b} + G^2} \right)}_{C_4(\beta_1,b)}
\end{align*}
and for all $\bm{\theta} \in \mathbb{R}^d$,
\begin{align*}
\limsup_{k \to + \infty} \mathbb{E}\left[
\nabla f(\bm{\theta}_k)^\top (\bm{\theta}_k - \bm{\theta})
\right]
&\leq
\underbrace{\frac{D(\bm{\theta}) M^{\frac{1}{4}}}{v_*^{\frac{1}{4}} \beta_1}\sqrt{ \frac{\sigma^2}{b} + G^2}}_{C_1(\beta_1,b)}
+
\underbrace{\frac{\alpha}{2 \sqrt{v_*} \beta_{1}} \left( \frac{\sigma^2}{b} + G^2 \right)}_{\bar{C}_2(\alpha,\beta_1,b)}\\
&\quad +
\underbrace{\frac{1 - \beta_{1}}{\beta_{1}}D(\bm{\theta}) G}_{C_3(\beta_1)}
+ 
\underbrace{\left(\frac{1}{\beta_1} + 2 (1 - \beta_{1}) \right)
D(\bm{\theta}) \left(B + \sqrt{\frac{\sigma^2}{b} + G^2} \right)}_{C_5(\beta_1,b)},
\end{align*}
where the parameters are defined as in Theorem \ref{c1}.
\end{cor}

The definitions of $C_1, \bar{C}_2, C_3, C_4$, and $C_5$ in Corollary \ref{cor:1} imply that using a small learning rate $\alpha$, $\beta_1$ close to $1$, and a large batch size $b$ is advantageous for Adam. 

\subsubsection{Diminishing learning rate rule}
Next, let us consider Adam defined by Algorithm \ref{algo:1} using the following diminishing learning rate $\alpha_k$ and hyperparameters $\beta_{1k}$ and $\beta_{2k}$ which converge to $1$: for all $k \geq 1$,
\begin{align}\label{diminishing_lr}
\alpha_{k} := \frac{1}{k^a}, \text{ }
\beta_{1k} := 1 - \frac{1}{k^{b_1}}, \text{ and }
\beta_{2k} := \left(1 - \frac{1}{k^{b_2}} \right)^{\frac{1}{k+1}},
\end{align}
where $\alpha_0 = 1$, $\beta_{10} = 0$, $\beta_{20} = 0$, and $a > 0, b_1 > 0$, and $b_2 > 0$ satisfy that 
\begin{align*}
a - b_1 + \frac{b_2}{2} > 0.
\end{align*}

The following is an analysis of Adam for a diminishing learning rate (see Appendix \ref{app:1} for the proof of Theorem \ref{d1}):

\begin{thm}\label{d1}
Suppose that (S1)--(S3) and (A1)--(A2) hold and $\bm{\theta}^\star \in \mathbb{R}^d$ is a stationary point of $f$. Then, Adam defined by Algorithm \ref{algo:1} using \eqref{diminishing_lr} satisfies that, for all $k\geq 1$, 
\begin{align*}
\mathbb{E}\left[
\bm{m}_{k}^\top (\bm{\theta}_k - \bm{\theta}^\star)
\right]
&\leq
\underbrace{\frac{D(\bm{\theta}^\star) M^{\frac{1}{4}} k^{b_1}}{v_*^{\frac{1}{4}}(k^{b_1} - 1)}\sqrt{ \frac{\sigma^2}{b} + G^2}}_{D_1(b_1,b,k)}
+
\underbrace{\frac{1}{2 \sqrt{v_*} (k^{b_1} - 1) k^{a +\frac{b_2}{2} -2 b_1}} \left( \frac{\sigma^2}{b} + G^2 \right)}_{D_2(a,b_1,b_2,b,k)}\\
&\quad + \underbrace{\frac{1}{k^{b_1} -1} D(\bm{\theta}^\star) G}_{D_3(b_1,k)} 
+ \underbrace{\frac{1}{k^{b_1}} D(\bm{\theta}^\star) \left( B + \sqrt{ \frac{\sigma^2}{b} + G^2} \right)}_{D_4(b_1,b,k)}
\end{align*}
and for all $\bm{\theta} \in \mathbb{R}^d$ and all $k\geq 1$,
\begin{align*}
\mathbb{E}\left[
\nabla f(\bm{\theta}_k)^\top (\bm{\theta}_k - \bm{\theta}) 
\right
]
&\leq
\underbrace{\frac{D(\bm{\theta}) M^{\frac{1}{4}} k^{b_1}}{v_*^{\frac{1}{4}}(k^{b_1} - 1)}\sqrt{ \frac{\sigma^2}{b} + G^2}}_{D_1(b_1,b,k)}
+
\underbrace{\frac{1}{2 \sqrt{v_*} (k^{b_1} - 1) k^{a +\frac{b_2}{2} -2 b_1}} \left( \frac{\sigma^2}{b} + G^2 \right)}_{D_2(a,b_1,b_2,b,k)}\\
&\quad + \underbrace{\frac{1}{k^{b_1} -1} D(\bm{\theta}) G}_{D_3(b_1,k)} 
+ \underbrace{\frac{1}{k^{b_1}} D(\bm{\theta}) \left( B + \sqrt{ \frac{\sigma^2}{b} + G^2} \right)}_{D_4(b_1,b,k)}\\
&\quad +
\underbrace{\frac{k^{{b_1}^2} + 2 k^{b_1} - 2}{k^{b_1} (k^{b_1} - 1)}
D(\bm{\theta}) \left(B + \sqrt{\frac{\sigma^2}{b} + G^2} \right)}_{D_5(b_1,b,k)},
\end{align*}
where the parameters are defined as in Theorem \ref{c1}.
\end{thm}

Theorem \ref{d1} leads to the following corollary: 

\begin{cor}\label{cor:2}
Under the assumptions in Theorem \ref{d1}, Adam defined by Algorithm \ref{algo:1} with \eqref{diminishing_lr} satisfies 
\begin{align*}
\limsup_{k\to +\infty} \mathbb{E}\left[
\bm{m}_{k}^\top (\bm{\theta}_k - \bm{\theta}^\star)
\right]
&\leq
\frac{D(\bm{\theta}^\star) M^{\frac{1}{4}}}{v_*^{\frac{1}{4}}}\sqrt{ \frac{\sigma^2}{b} + G^2}
\end{align*}
and for all $\bm{\theta} \in \mathbb{R}^d$ and all $k\geq 1$,
\begin{align*}
\limsup_{k \to + \infty} \mathbb{E}\left[
\nabla f(\bm{\theta}_k)^\top (\bm{\theta}_k - \bm{\theta}) 
\right
]
&\leq
\frac{D(\bm{\theta}) M^{\frac{1}{4}}}{v_*^{\frac{1}{4}}}\sqrt{ \frac{\sigma^2}{b} + G^2}
+
D(\bm{\theta}) \left(B + \sqrt{\frac{\sigma^2}{b} + G^2} \right),
\end{align*}
where the parameters are defined as in Theorem \ref{c1}.
\end{cor}

Corollary \ref{cor:2} implies that the upper bounds of the performance measures of Adam using $(\alpha_k)_{k\in\mathbb{N}}$ converging to $0$ and $(\beta_{1k})_{k\in\mathbb{N}}$ and $(\beta_{2k})_{k\in\mathbb{N}}$ converging to $1$ are smaller than the ones of Adam using constant parameters $\alpha$, $\beta_1$, and $\beta_2$ shown in Corollary \ref{cor:1}.

\subsection{Adam with Condition \eqref{max_1}}
\label{subsec:3.2}
As described in Theorems \ref{c1} and \ref{d1} and Corollaries \ref{cor:1} and \ref{cor:2}, although a small learning rate $\alpha_k$ and hyperparameters $\beta_1$ and $\beta_1$ close to $1$ are useful for implementing Adam, the upper bounds of the performance measures \eqref{pm_1} and \eqref{pm_3} in Theorems \ref{c1} and \ref{d1} and Corollaries \ref{cor:1} and \ref{cor:2} would not be small. This result can be attributed to that, in general, Adam does not converge to a local minimizer of the problem of minimizing $f$ over $\mathbb{R}^d$. In fact, there is a stochastic convex optimization problem in which Adam using $\beta_1 < \sqrt{\beta_2}$ (e.g., $\beta_1 = 0.9$ and $\beta_2 = 0.999$) does not converge to the optimal solution \cite[Theorem 3]{reddi2018}.

Theorem 4 in \cite{reddi2018} showed that Adam with \eqref{max_1} and \eqref{dim_1} can solve convex stochastic optimization problems (see also \cite[Theorems 2.1 and 2.2]{ada} for the analyses of AdaBelief for convex and nonconvex optimization). Motivated by the results in \cite{reddi2018,ada}, we will analyze Adam with \eqref{max_1}.

\subsubsection{Constant learning rate rule}
The following is an analysis of Adam using \eqref{max_1}, a constant learning rate, and constant hyperparameters (see Appendix \ref{app:1} for the proof of Theorem \ref{c2}).

\begin{thm}\label{c2}
Suppose that (S1)--(S3) and (A1)--(A2) hold and $\bm{\theta}^\star \in \mathbb{R}^d$ is a stationary point of $f$. Then, Adam defined by Algorithm \ref{algo:1} using \eqref{max_1} and \eqref{constant_lr} satisfies that, for all $K \geq 1$,
\begin{align*}
\frac{1}{K} \sum_{k=1}^K
\mathbb{E}\left[
\bm{m}_k^\top (\bm{\theta}_k - \bm{\theta}^\star) 
\right]
&\leq
\underbrace{\frac{d \tilde{D}(\bm{\theta}^\star) \sqrt{M}(1-\beta_1^{K+1})}{2 \alpha \beta_1 \sqrt{1 - \beta_2} K}}_{\bar{C}_1 (\alpha,\beta_1,\beta_2,K)}
+
\underbrace{\frac{\alpha}{2 \sqrt{v_*} \beta_{1} (1-\beta_1)}
\left(\frac{\sigma^2}{b} + G^2 \right)}_{\hat{C}_2(\alpha,\beta_1,b)}\\
&\quad 
+
\underbrace{\frac{1 - \beta_{1}}{\beta_{1}} D(\bm{\theta}^\star) G }_{C_3(\beta_1)}
+ \underbrace{(1-\beta_1) D(\bm{\theta}^\star) \left( B + \sqrt{\frac{\sigma^2}{b} + G^2} \right)}_{C_4(\beta_1,b)}
\end{align*}
for all $\bm{\theta} \in \mathbb{R}^d$ and all $K \geq 1$,
\begin{align*}
\frac{1}{K} \sum_{k=1}^K
\mathbb{E}\left[
\nabla f(\bm{\theta}_k)^\top (\bm{\theta}_k - \bm{\theta}) 
\right]
&\leq
\underbrace{\frac{d \tilde{D}(\bm{\theta}) \sqrt{M}(1-\beta_1^{K+1})}{2 \alpha \beta_1 \sqrt{1 - \beta_2} K}}_{\bar{C}_1 (\alpha,\beta_1,\beta_2,K)}
+
\underbrace{\frac{\alpha}{2 \sqrt{v_*} \beta_{1} (1-\beta_1)}
\left(\frac{\sigma^2}{b} + G^2 \right)}_{\hat{C}_2(\alpha,\beta_1,b)}\\ 
&\quad
+
\underbrace{\frac{1 - \beta_{1}}{\beta_{1}} D(\bm{\theta}) G}_{C_3(\beta_1)}
+
\underbrace{\left( \frac{1}{\beta_{1}} + 2 (1 - \beta_{1}) \right)
D(\bm{\theta})
\left( B + \sqrt{\frac{\sigma^2}{b} + G^2} \right)}_{C_5(\beta_1,b)}, 
\end{align*}
where the parameters are defined as in Theorem \ref{c1} and $\tilde{D} (\bm{\theta}) := \sup \{ \max_{i\in [d]} (\theta_{k,i} - \theta_i)^2 \colon k \in \mathbb{N} \} < + \infty$.
\end{thm}

Let us compare Theorem \ref{c1} with Theorem \ref{c2}. A significant difference is 
\begin{align*}
C_1 (\beta_1,b) := \frac{D(\bm{\theta}^\star) M^{\frac{1}{4}}}{v_*^{\frac{1}{4}} \beta_1}\sqrt{ \frac{\sigma^2}{b} + G^2} \text{ and }
\bar{C}_1 (\alpha,\beta_1,\beta_2,K) := 
\frac{d D(\bm{\theta}^\star) \sqrt{M}(1-\beta_1^{K+1})}{2 \alpha \beta_1 \sqrt{1 - \beta_2} K}.
\end{align*}
Although $\bar{C}_1$ in Theorem \ref{c2} is monotone increasing for $\beta_2 \in [0,1)$ and $1/\alpha$, $\bar{C}_1$ becomes small when $K$ is sufficiently large. $C_1$ in Theorem \ref{cor:1} does not change for any $K$. $\hat{C}_2$ in Theorem \ref{c2} defined by 
\begin{align*}
\hat{C}_2 (\alpha, \beta_1,b) := 
\frac{\alpha}{2 \sqrt{v_*} \beta_{1} (1-\beta_1)}
\left(\frac{\sigma^2}{b} + G^2 \right)
\end{align*}
is monotone increasing for $\beta_1 \geq 1/2$ (see also the discussion in \eqref{tilde_c_2}). Hence, we must set a small learning rate $\alpha$ to make $\hat{C}_2$ small when $\beta_1$ is close to $1$. From the definition of $\bar{C}_1$, the use of a small learning rate $\alpha$ entails a large number of steps $K$. Therefore, Theorem \ref{c2} indicates that, if $K$ is sufficiently large, then Adam under Condition \eqref{max_1} has a tighter upper bound of $(1/K) \sum_{k=1}^K \mathbb{E}[ \bm{m}_k^\top (\bm{\theta}_k - \bm{\theta}^\star)]$ than Adam without Condition \eqref{max_1}. Theorem \ref{c1} implies that there exist $C_i$ ($i=1,2,3$) such that, for all $K \geq 1$,
\begin{align}\label{key_1}
\frac{1}{K} \sum_{k=1}^K
\mathbb{E}\left[
\bm{m}_k^\top (\bm{\theta}_k - \bm{\theta}^\star) 
\right]
=
\mathcal{O}\left( \frac{1}{K} \right)
+ C_1 \alpha + C_2 \frac{\alpha}{b} + C_3 \frac{1-\beta_1}{\beta_1}.
\end{align}

Here, we compare Theorem 1 in \cite{NEURIPS2018_90365351} with Theorem \ref{c2} in this paper. Theorem 1 and the proof of Theorem 1 in \cite{NEURIPS2018_90365351} show that, under the condition that $\nabla f$ is Lipschitz continuous with the Lipschitz constant $L$, Adam using $\alpha = \mathcal{O}(1/L)$ and $\beta_1 = 0$ satisfies 
\begin{align*}
\frac{1}{K} \sum_{k=1}^K \mathbb{E}\left[ \|\nabla f (\bm{\theta}_k)\|^2 \right]
\leq 2 \left(\sqrt{\beta_2} G + \epsilon \right)
\left\{
\frac{f(\bm{\theta}_1) - f (\bm{\theta}^\star)}{\alpha K}
+
\left(\frac{G\sqrt{1-\beta_2}}{\epsilon^2}
+ 
\frac{L \alpha}{2 \epsilon^2}
\right)
\frac{\sigma^2}{b}
\right\},
\end{align*}
that is,
\begin{align}\label{zaheer}
\frac{1}{K} \sum_{k=1}^K \mathbb{E}\left[ \|\nabla f (\bm{\theta}_k)\|^2 \right]
=
\mathcal{O} \left( \frac{1}{K}\right)
+ C_4 \frac{\alpha}{b} + C_5 \left(\sqrt{\beta_2} + 1 \right)\sqrt{1-\beta_2},
\end{align}
where $\epsilon > 0$ is the precision for nonconvex optimization, and $C_4$ and $C_5$ are constants. Theorem 1 in \cite{NEURIPS2018_90365351} assumes that $f$ is Lipschitz smooth, while Theorem \ref{c2} in this paper does not assume so. Hence, there is a difference in the performance measure between \eqref{key_1} and \eqref{zaheer}. However, \eqref{key_1} and \eqref{zaheer} show that using a small learning rate $\alpha$, hyperparameters close to $1$, and a large batch size $b$ is advantageous for Adam.

Let us compare the results in \cite{sca2020} with Theorem \ref{c2}. Theorem 12 in \cite{sca2020} indicates that SGD using a constant learning rate $\alpha = 1/L$, where $L > 0$ is the Lipschitz constant of the Lipschitz continuous gradient $\nabla f$, almost surely satisfies
\begin{align*}
\frac{1}{K} \sum_{k=1}^K \| \nabla f(\bm{\theta}_k) \|^2 
\leq \frac{2 L (f(\bm{\theta}_0) - f^\star)}{K} + \sigma^2,
\end{align*} 
where $f^\star := \min_{\bm{\theta} \in \mathbb{R}^d} f(\bm{\theta})$. Theorem \ref{c2} indicates that Adam using a constant learning rate $\alpha$ satisfies 
\begin{align*}
\frac{1}{K} \sum_{k=1}^K
\mathbb{E}\left[
\nabla f(\bm{\theta}_k)^\top (\bm{\theta}_k - \bm{\theta}) 
\right]
\leq \frac{d D(\bm{\theta}) \sqrt{M}(1-\beta_1^{K+1})}{2 \alpha \beta_1 \sqrt{1 - \beta_2}K}
+ \hat{C}_2(\alpha,\beta_1,b) + C_3(\beta_1) + C_5(\beta_1,b)
\end{align*}
without assuming that $\nabla f$ is Lipschitz continuous. This result indicates that using a small learning rate $\alpha$, a hyperparameter $\beta_1$ close to $1$, and a large batch size $b$ is advantageous for Adam.

$\hat{C}_2$ in Theorem \ref{c2} is obtained under conditions such that $\alpha_k = \alpha$ and $\beta_{1k} = \beta_1$ and 
\begin{align*}
\frac{(\sigma^2 b^{-1} + G^2)}{2 \sqrt{v_*} K}
\sum_{k=1}^K \frac{\alpha_k \sqrt{1 - \beta_{2k}^{k+1}}}{\beta_{1k} (1 - \beta_{1k}^{k+1})}
\leq
\frac{\alpha (\sigma^2 b^{-1} + G^2)}{2 \sqrt{v_*} \beta_1 (1-\beta_1) K}
\sum_{k=1}^K \sqrt{1 - \beta_{2k}^{k+1}}
\leq
\underbrace{\frac{\alpha (\sigma^2 b^{-1} + G^2)}{2 \sqrt{v_*} \beta_1 (1-\beta_1)}}_{\hat{C}_2 (\alpha,\beta_1,b)},
\end{align*}
where the first inequality comes from $1 - \beta_1 \leq 1 - \beta_1^{k+1}$ ($k \in \mathbb{N}$) and the second inequality comes from $1 - \beta_{2k} \leq 1$ ($k \in \mathbb{N}$). Hence, we can improve $\hat{C}_2$ by using the property of $\beta_{2k}$. Let us consider Adam defined by Algorithm \ref{algo:1} using the following constant learning rate and hyperparameters:
\begin{align}\label{constant_lr_1}
\alpha_{k} := \alpha \in (0,+\infty), \text{ }
\beta_{1k} := \beta_1 \in (0,1), \text{ and }
\beta_{2k} := \left(1 - \frac{1}{k^{b_2}}\right)^{\frac{1}{k+1}},
\end{align}
where $\beta_{20} = 0$ and $b_2 \in (0,2)$.

\begin{thm}\label{c3}
Suppose that (S1)--(S3) and (A1)--(A2) hold and $\bm{\theta}^\star \in \mathbb{R}^d$ is a stationary point of $f$. Then, Adam defined by Algorithm \ref{algo:1} using \eqref{max_1} and \eqref{constant_lr_1} satisfies that, for all $K \geq 1$,
\begin{align*}
\frac{1}{K} \sum_{k=1}^K
\mathbb{E}\left[
\bm{m}_k^\top (\bm{\theta}_k - \bm{\theta}^\star) 
\right]
&\leq
\underbrace{\frac{d \tilde{D}(\bm{\theta}^\star) \sqrt{M}}{2 \alpha \beta_1 K^{1 - \frac{b_2}{2}}}}_{\hat{C}_1 (\alpha,\beta_1,\beta_2,K)}
+
\underbrace{\frac{\alpha}{\sqrt{v_*} \beta_{1} (1-\beta_1) K^{ \frac{b_2}{2}}} \left(\frac{\sigma^2}{b} + G^2 \right)}_{\tilde{C}_2(\alpha,\beta_1,b,K)}\\
&\quad 
+
\underbrace{\frac{1 - \beta_{1}}{\beta_{1}} D(\bm{\theta}^\star) G }_{C_3(\beta_1)}
+ \underbrace{(1-\beta_1) D(\bm{\theta}^\star) \left( B + \sqrt{\frac{\sigma^2}{b} + G^2} \right)}_{C_4(\beta_1,b)}
\end{align*}
for all $\bm{\theta} \in \mathbb{R}^d$ and all $K \geq 1$,
\begin{align*}
\frac{1}{K} \sum_{k=1}^K
\mathbb{E}\left[
\nabla f(\bm{\theta}_k)^\top (\bm{\theta}_k - \bm{\theta}) 
\right]
&\leq
\underbrace{\frac{d \tilde{D}(\bm{\theta}) \sqrt{M}}{2 \alpha \beta_1 K^{1 - \frac{b_2}{2}}}}_{\hat{C}_1 (\alpha,\beta_1,\beta_2,K)}
+
\underbrace{\frac{\alpha}{\sqrt{v_*} \beta_{1} (1-\beta_1) K^{\frac{b_2}{2}}}\left(\frac{\sigma^2}{b} + G^2 \right)}_{\tilde{C}_2(\alpha,\beta_1,b,K)}\\ 
&\quad
+
\underbrace{\frac{1 - \beta_{1}}{\beta_{1}} D(\bm{\theta}) G}_{C_3(\beta_1)}
+
\underbrace{\left( \frac{1}{\beta_{1}} + 2 (1 - \beta_{1}) \right)
D(\bm{\theta})
\left( B + \sqrt{\frac{\sigma^2}{b} + G^2} \right)}_{C_5(\beta_1,b)}, 
\end{align*}
where the parameters are defined as in Theorem \ref{c2}.
\end{thm}

The definitions of $\hat{C}_1$ and $\tilde{C}_2$ imply that the best setting of $b_2$ is $1$, since
\begin{align*}
0 < \min \left\{ 1 - \frac{b_2}{2}, \frac{b_2}{2} \right\}
\end{align*}
attains the maximum value $1/2$ when $b_2 = 1$. Hence, Adam using \eqref{max_1} and \eqref{constant_lr_1} with $b_2 = 1$ satisfies 
\begin{align*}
\frac{1}{K} \sum_{k=1}^K
\mathbb{E}\left[
\bm{m}_k^\top (\bm{\theta}_k - \bm{\theta}^\star) 
\right]
= 
\mathcal{O}\left( \frac{1}{\sqrt{K}} + \frac{1-\beta_1}{\beta_1} \right).
\end{align*}

\subsubsection{Diminishing learning rate rule}
Let us consider Adam defined by Algorithm \ref{algo:1} under Condition \eqref{max_1} and the following diminishing learning rate $\alpha_k$ and constant hyperparameters $\beta_{1}$ and $\beta_{2}$: for all $k \geq 1$,
\begin{align}\label{diminishing_lr_1}
\alpha_{k} := \frac{1}{k^a}, \text{ }
\beta_{1k} := \beta_1 \in (0,1), \text{ and }
\beta_{2k} := \beta_2 \in [0,1),
\end{align}
where $\alpha_0 = 1$ and $a \in (0,1)$. 

\begin{thm}\label{d2}
Suppose that (S1)--(S3) and (A1)--(A2) hold and $\bm{\theta}^\star \in \mathbb{R}^d$ is a stationary point of $f$. Then, Adam defined by Algorithm \ref{algo:1} using \eqref{max_1} and \eqref{diminishing_lr_1} satisfies that, for all $K \geq 1$,
\begin{align*}
\frac{1}{K} \sum_{k=1}^K
\mathbb{E}\left[
\bm{m}_k^\top (\bm{\theta}_k - \bm{\theta}^\star) 
\right]
&\leq
\underbrace{\frac{d \tilde{D}(\bm{\theta}^\star) \sqrt{M}}{2 \beta_1 \sqrt{1 - \beta_2} K^{1 - a}}}_{\bar{D}_1 (\beta_1,\beta_2,K)}
+
\underbrace{\frac{1}{\sqrt{v_*} \beta_{1} (1-\beta_1) K^a}
\left(\frac{\sigma^2}{b} + G^2 \right)}_{\bar{D}_2(\beta_1,b,K)}\\
&\quad 
+
\underbrace{\frac{1 - \beta_{1}}{\beta_{1}} D(\bm{\theta}^\star) G }_{C_3(\beta_1)}
+ \underbrace{(1-\beta_1) D(\bm{\theta}^\star) \left( B + \sqrt{\frac{\sigma^2}{b} + G^2} \right)}_{C_4(\beta_1,b)}
\end{align*}
for all $\bm{\theta} \in \mathbb{R}^d$ and all $K \geq 1$,
\begin{align*}
\frac{1}{K} \sum_{k=1}^K
\mathbb{E}\left[
\nabla f(\bm{\theta}_k)^\top (\bm{\theta}_k - \bm{\theta}) 
\right]
&\leq
\underbrace{\frac{d \tilde{D}(\bm{\theta}) \sqrt{M}}{2 \beta_1 \sqrt{1 - \beta_2} K^{1 - a}}}_{\bar{D}_1 (\beta_1,\beta_2,K)}
+
\underbrace{\frac{1}{\sqrt{v_*} \beta_{1} (1-\beta_1) K^a}\left(\frac{\sigma^2}{b} + G^2 \right)}_{\bar{D}_2(\beta_1,b,K)}\\ 
&\quad
+
\underbrace{\frac{1 - \beta_{1}}{\beta_{1}} D(\bm{\theta}) G}_{C_3(\beta_1)}
+
\underbrace{\left( \frac{1}{\beta_{1}} + 2 (1 - \beta_{1}) \right)
D(\bm{\theta})
\left( B + \sqrt{\frac{\sigma^2}{b} + G^2} \right)}_{C_5(\beta_1,b)}, 
\end{align*}
where the parameters are defined as in Theorem \ref{c2}.
\end{thm}

The definitions of $\bar{D}_1$ and $\bar{D}_2$ imply that the best setting of $a$ is $1/2$ since
\begin{align*}
\min \left\{ 1 - a, a \right\}
\end{align*}
attains the maximum value $1/2$ when $a = 1/2$. Hence, Adam using \eqref{max_1} and \eqref{diminishing_lr_1} with $a = 1/2$ satisfies 
\begin{align*}
\frac{1}{K} \sum_{k=1}^K
\mathbb{E}\left[
\bm{m}_k^\top (\bm{\theta}_k - \bm{\theta}^\star) 
\right]
= 
\mathcal{O}\left( \frac{1}{\sqrt{K}} + \frac{1-\beta_1}{\beta_1} \right).
\end{align*}

Motivated by \eqref{constant_lr_1}, we consider Adam defined by Algorithm \ref{algo:1} using the following constant learning rate and hyperparameters:
\begin{align}\label{diminishing_lr_2}
\alpha_{k} := \frac{1}{k^a}, \text{ }
\beta_{1k} := \beta_1 \in (0,1), \text{ and }
\beta_{2k} := \left(1 - \frac{1}{k^{b_2}}\right)^{\frac{1}{k+1}},
\end{align}
where $\alpha_0 = 0$, $\beta_{20} = 0$, and $a \in (0,1)$ and $b_2 \in (0,2)$ satisfy 
\begin{align*}
1 - a - \frac{b_2}{2} > 0.
\end{align*}

\begin{thm}\label{d3}
Suppose that (S1)--(S3) and (A1)--(A2) hold and $\bm{\theta}^\star \in \mathbb{R}^d$ is a stationary point of $f$. Then, Adam defined by Algorithm \ref{algo:1} using \eqref{max_1} and \eqref{diminishing_lr_2} satisfies that, for all $K \geq 1$,
\begin{align*}
\frac{1}{K} \sum_{k=1}^K
\mathbb{E}\left[
\bm{m}_k^\top (\bm{\theta}_k - \bm{\theta}^\star) 
\right]
&\leq
\underbrace{\frac{d \tilde{D}(\bm{\theta}^\star) \sqrt{M}}{2 \beta_1 K^{1 - (a + \frac{b_2}{2})}}}_{\hat{D}_1 (\beta_1,\beta_2,K)}
+
\underbrace{\frac{(\sigma^2 b^{-1} + G^2 )}{\sqrt{v_*} \beta_{1} (1-\beta_1) K^{a + \frac{b_2}{2}}}}_{\hat{D}_2(\beta_1,b,K)}\\
&\quad 
+
\underbrace{\frac{1 - \beta_{1}}{\beta_{1}} D(\bm{\theta}^\star) G }_{C_3(\beta_1)}
+ \underbrace{(1-\beta_1) D(\bm{\theta}^\star) \left( B + \sqrt{\frac{\sigma^2}{b} + G^2} \right)}_{C_4(\beta_1,b)}
\end{align*}
for all $\bm{\theta} \in \mathbb{R}^d$ and all $K \geq 1$,
\begin{align*}
\frac{1}{K} \sum_{k=1}^K
\mathbb{E}\left[
\nabla f(\bm{\theta}_k)^\top (\bm{\theta}_k - \bm{\theta}) 
\right]
&\leq
\underbrace{\frac{d \tilde{D}(\bm{\theta}) \sqrt{M}}{2 \beta_1 K^{1 - (a + \frac{b_2}{2})}}}_{\hat{D}_1 (\beta_1,\beta_2,K)}
+
\underbrace{\frac{(\sigma^2 b^{-1} + G^2 )}{\sqrt{v_*} \beta_{1} (1-\beta_1) K^{a + \frac{b_2}{2}}}}_{\hat{D}_2(\beta_1,b,K)}\\
&\quad
+
\underbrace{\frac{1 - \beta_{1}}{\beta_{1}} D(\bm{\theta}) G}_{C_3(\beta_1)}
+
\underbrace{\left( \frac{1}{\beta_{1}} + 2 (1 - \beta_{1}) \right)
D(\bm{\theta})
\left( B + \sqrt{\frac{\sigma^2}{b} + G^2} \right)}_{C_5(\beta_1,b)}, 
\end{align*}
where the parameters are defined as in Theorem \ref{c2}.
\end{thm}

The definitions of $\hat{D}_1$ and $\hat{D}_2$ imply that 
\begin{align*}
\min \left\{ 1 - \left(a + \frac{b_2}{2} \right), a + \frac{b_2}{2} \right\}
\end{align*}
attains the maximum value $1/2$ when $a + b_2/2 = 1/2$, e.g., $a = 1/4$ and $b_2 = 1/2$. Hence, Adam using \eqref{max_1} and \eqref{diminishing_lr_2} with $a = 1/2$ satisfies
\begin{align*}
\frac{1}{K} \sum_{k=1}^K
\mathbb{E}\left[
\bm{m}_k^\top (\bm{\theta}_k - \bm{\theta}^\star) 
\right]
= 
\mathcal{O}\left( \frac{1}{\sqrt{K}} + \frac{1-\beta_1}{\beta_1} \right).
\end{align*}

\section{Conclusion and Future Work}
\label{sec:4}
This paper presented theoretical analyses of the Adam optimizer without assuming the Lipschitz smoothness condition for nonconvex optimization in deep learning. The analyzes indicated that Adam performs well when it uses hyperparameters close to one and not only a small learning rate but also a diminishing learning rate. Hence, our results are theoretical evidence supporting numerical evaluations showing that small constant learning rates and hyperparameters close to one are advantageous for training deep neural networks. 

This paper focused on convergence analyses of Adam for nonconvex optimization. In the future, we should consider developing convergence analyses of Adam's variants for nonconvex optimization and show theoretically that adaptive methods, such as Yogi, AMSGrad, AdaBelief, Padam, and AdamW, using hyperparameters close to one perform well.

\appendix
\section{Appendix A}
\label{app:1}
Unless stated otherwise, all relationships between random variables are supported to hold almost surely. Let $S$ be a positive definite matrix denoted by $S \in \mathbb{S}_{++}^d$. The $S$-inner product of $\mathbb{R}^d$ is defined for all $\bm{x}, \bm{y} \in \mathbb{R}^d$ by $\langle \bm{x},\bm{y} \rangle_S := \langle \bm{x}, S \bm{y} \rangle = \bm{x}^\top (S \bm{y})$, and the $S$-norm is defined by $\|\bm{x}\|_S := \sqrt{\langle \bm{x}, S \bm{x} \rangle}$.

\subsection{Lemmas}
\begin{lem}\label{lem:1}
Suppose that (S1), (S2)\eqref{gradient}, and (S3) hold. Then, Adam defined by Algorithm \ref{algo:1} satisfies the following: for all $k\in\mathbb{N}$ and all $\bm{\theta} \in \mathbb{R}^d$,
\begin{align*}
\mathbb{E}\left[ \left\| \bm{\theta}_{k+1} - \bm{\theta} \right\|_{\mathsf{H}_k}^2 \right]
&=
\mathbb{E}\left[\left\| \bm{\theta}_{k} - \bm{\theta} \right\|_{\mathsf{H}_k}^2 \right]
+ \alpha_k^2 \mathbb{E}\left[\left\| \bm{\mathsf{d}}_k \right\|_{\mathsf{H}_k}^2 \right]\\
&\quad + 2 \alpha_k \left\{
\frac{\beta_{1k}}{\tilde{\beta}_{1k}} 
\mathbb{E}\left[ (\bm{\theta} - \bm{\theta}_k)^\top \bm{m}_{k-1} \right] 
+\frac{\hat{\beta}_{1k}}{\tilde{\beta}_{1k}} 
\mathbb{E}\left[ (\bm{\theta} - \bm{\theta}_k)^\top \nabla f (\bm{\theta}_k) \right]
\right\},
\end{align*}
where $\bm{\mathsf{d}}_k := - \mathsf{H}_k^{-1} \hat{\bm{m}}_k$, $\hat{\beta}_{1k} := 1 - \beta_{1k}$, and $\tilde{\beta}_{1k} := 1 - \beta_{1k}^{k+1}$.
\end{lem}

\begin{proof}
Let $\bm{\theta} \in \mathbb{R}^d$ and $k\in\mathbb{N}$. The definition of $\bm{\theta}_{k+1} := \bm{\theta}_{k} + \alpha_k \bm{\mathsf{d}}_k$ implies that 
\begin{align*}
\| \bm{\theta}_{k+1} - \bm{\theta} \|_{\mathsf{H}_k}^2
= 
\| \bm{\theta}_{k} - \bm{\theta} \|_{\mathsf{H}_k}^2
+ 2 \alpha_k \langle \bm{\theta}_{k} - \bm{\theta}, \bm{\mathsf{d}}_k \rangle_{\mathsf{H}_k}
+ \alpha_k^2 \|\bm{\mathsf{d}}_k\|_{\mathsf{H}_k}^2.
\end{align*}
Moreover, the definitions of $\bm{\mathsf{d}}_k$, $\bm{m}_k$, and $\hat{\bm{m}}_k$ ensure that 
\begin{align*}
\left\langle \bm{\theta}_k - \bm{\theta}, \bm{\mathsf{d}}_k \right\rangle_{\mathsf{H}_k}
&=
\left\langle \bm{\theta}_k - \bm{\theta}, \mathsf{H}_k \bm{\mathsf{d}}_k \right\rangle
=
\left\langle \bm{\theta} - \bm{\theta}_k, \hat{\bm{m}}_k \right\rangle
=
\frac{1}{\tilde{\beta}_{1k}}
(\bm{\theta} - \bm{\theta}_k)^\top {\bm{m}}_k\\ 
&=
\frac{\beta_{1k}}{\tilde{\beta}_{1k}} 
(\bm{\theta} - \bm{\theta}_k)^\top \bm{m}_{k-1} 
+
\frac{\hat{\beta}_{1k}}{\tilde{\beta}_{1k}} 
(\bm{\theta} - \bm{\theta}_k)^\top \nabla f_{B_k}(\bm{\theta}_k).
\end{align*}
Hence, 
\begin{align}\label{ineq:004}
\begin{split}
\left\|\bm{\theta}_{k+1} - \bm{\theta} \right\|_{\mathsf{H}_k}^2
&=
\left\| \bm{\theta}_k -\bm{\theta} \right\|_{\mathsf{H}_k}^2
+ \alpha_k^2 \left\| \bm{\mathsf{d}}_k \right\|_{\mathsf{H}_k}^2\\
&\quad + 2 \alpha_k \left\{
\frac{\beta_{1k}}{\tilde{\beta}_{1k}} 
(\bm{\theta} - \bm{\theta}_k)^\top \bm{m}_{k-1} 
+ \frac{\hat{\beta}_{1k}}{\tilde{\beta}_{1k}} 
(\bm{\theta} - \bm{\theta}_k)^\top \nabla f_{B_k} (\bm{\theta}_k) 
\right\}.
\end{split}
\end{align}
Conditions \eqref{gradient} and (S3) guarantee that
\begin{align*}
\mathbb{E}\left[ \mathbb{E} \left[(\bm{\theta} - \bm{\theta}_k)^\top \nabla f_{B_k} (\bm{\theta}_k) \Big| \bm{\theta}_k \right] \right]
=
\mathbb{E} \left[(\bm{\theta} - \bm{\theta}_k)^\top 
\mathbb{E} \left[\nabla f_{B_k} (\bm{\theta}_k) \Big| \bm{\theta}_k \right] \right]
=
\mathbb{E} \left[(\bm{\theta} - \bm{\theta}_k)^\top 
\nabla f (\bm{\theta}_k) \right].
\end{align*}
Therefore, the lemma follows by taking the expectation on both sides of \eqref{ineq:004}. This completes the proof.
\end{proof}

\begin{lem}\label{lem:bdd}
Adam defined by Algorithm \ref{algo:1} satisfies that, under (S2)\eqref{gradient}, \eqref{sigma}, and (A1), for all $k\in\mathbb{N}$,
\begin{align*}
\mathbb{E}\left[ \left\|\bm{m}_k \right\|^2 \right] 
\leq 
\frac{\sigma^2}{b} + G^2, \quad 
\mathbb{E}\left[ \left\|\bm{\mathsf{d}}_k \right\|_{\mathsf{H}_k}^2 \right] 
\leq 
\frac{\sqrt{\tilde{\beta}_{2k}}}{\tilde{\beta}_{1k}^2 \sqrt{{v}_*}} \left( \frac{\sigma^2}{b} + G^2 \right),
\end{align*}
where ${v}_* := \inf \{ \min_{i\in [d]} {v}_{k,i} \colon k\in \mathbb{N}\}$, $\tilde{\beta}_{1k} := 1 - \beta_{1k}^{k+1}$, and $\tilde{\beta}_{2k} := 1 - \beta_{2k}^{k+1}$.
\end{lem}

\begin{proof}
Assumption (S2)\eqref{gradient} implies that
\begin{align}\label{equation}
\begin{split}
\mathbb{E} \left[ \left\| \nabla f_{B_k} (\bm{\theta}_{k}) \right\|^2
\Big| \bm{\theta}_k
\right]
&=
\mathbb{E} \left[\left\| \nabla f_{B_k} (\bm{\theta}_{k}) 
- \nabla f (\bm{\theta}_{k}) + \nabla f (\bm{\theta}_{k}) \right\|^2
\Big| \bm{\theta}_k
\right]\\
&=
\mathbb{E} \left[ \left\| \nabla f_{B_k} (\bm{\theta}_{k}) 
- \nabla f (\bm{\theta}_{k}) \right\|^2 \Big| \bm{\theta}_k
\right]
+ 
\mathbb{E} \left[ \left\| \nabla f (\bm{\theta}_{k}) \right\|^2 \Big| \bm{\theta}_k
\right]\\
&\quad + 2 
\mathbb{E} \left[ 
(\nabla f_{B_k} (\bm{\theta}_{k}) 
- \nabla f (\bm{\theta}_{k}))^\top \nabla f (\bm{\theta}_{k})
\Big| \bm{\theta}_k \right]\\
&= 
\mathbb{E} \left[\left\| \nabla f_{B_k} (\bm{\theta}_{k}) 
- \nabla f (\bm{\theta}_{k}) \right\|^2 \Big| \bm{\theta}_k
\right]
+ 
\| \nabla f (\bm{\theta}_{k}) \|^2,
\end{split}
\end{align} 
which, together with (S2)\eqref{sigma} and (A1), implies that 
\begin{align}\label{A3}
\mathbb{E} \left[ \left\| \nabla f_{B_k} (\bm{\theta}_{k}) \right\|^2
\right]
\leq 
\frac{\sigma^2}{b} + G^2.
\end{align}
The convexity of $\|\cdot\|^2$, together with the definition of $\bm{m}_k$ and \eqref{A3}, guarantees that, for all $k\in\mathbb{N}$,
\begin{align*}
\mathbb{E}\left[ \left\|\bm{m}_k \right\|^2 \right]
&\leq \beta_{1k} \mathbb{E}\left[ \left\|\bm{m}_{k-1} \right\|^2 \right] + 
\hat{\beta}_{1k} \mathbb{E}\left[ \left\|\nabla f_{B_k} (\bm{\theta}_k) \right\|^2 \right]\\
&\leq 
\beta_{1k} \mathbb{E} \left[ \left\|\bm{m}_{k-1} \right\|^2 \right] + \hat{\beta}_{1k} 
\left( \frac{\sigma^2}{b} + G^2 \right).
\end{align*}
Induction thus ensures that, for all $k\in\mathbb{N}$,
\begin{align}\label{induction}
\mathbb{E} \left[ \left\|\bm{m}_k \right\|^2 \right] \leq 
\max \left\{ \|\bm{m}_{-1}\|^2, \frac{\sigma^2}{b} + G^2 \right\} 
= \frac{\sigma^2}{b} + G^2,
\end{align}
where $\bm{m}_{-1} = \bm{0}$. For $k\in\mathbb{N}$, $\mathsf{H}_k \in \mathbb{S}_{++}^d$ guarantees the existence of a unique matrix $\overline{\mathsf{H}}_k \in \mathbb{S}_{++}^d$ such that $\mathsf{H}_k = \overline{\mathsf{H}}_k^2$ \cite[Theorem 7.2.6]{horn}. We have that, for all $\bm{x}\in\mathbb{R}^d$, $\|\bm{x}\|_{\mathsf{H}_k}^2 = \| \overline{\mathsf{H}}_k \bm{x} \|^2$. Accordingly, the definitions of $\bm{\mathsf{d}}_k$ and $\hat{\bm{m}}_k$ imply that, for all $k\in\mathbb{N}$, 
\begin{align*}
\mathbb{E} \left[ \left\| \bm{\mathsf{d}}_k \right\|_{\mathsf{H}_k}^2 \right]
= 
\mathbb{E} \left[ \left\| \overline{\mathsf{H}}_k^{-1} \mathsf{H}_k\bm{\mathsf{d}}_k \right\|^2 \right]
\leq 
\frac{1}{\tilde{\beta}_{1k}^2} \mathbb{E} \left[ \left\| \overline{\mathsf{H}}_k^{-1} \right\|^2 \|\bm{m}_k \|^2 \right],
\end{align*}
where 
\begin{align*}
\left\| \overline{\mathsf{H}}_k^{-1} \right\| 
= \left\| \mathsf{diag}\left(\hat{v}_{k,i}^{-\frac{1}{4}} \right) \right\| 
= \max_{i\in [d]} \hat{v}_{k,i}^{-\frac{1}{4}}
= \max_{i\in [d]} \left( \frac{v_{k,i}}{\tilde{\beta}_{2k}} \right)^{-\frac{1}{4}} 
=:
\left( \frac{v_{k,i^*}}{\tilde{\beta}_{2k}} \right)^{-\frac{1}{4}}. 
\end{align*} 
Moreover, the definition of 
\begin{align*}
{v}_* := \inf \left\{ {v}_{k,i^*} \colon k\in \mathbb{N} \right\}
\end{align*} 
and (\ref{induction}) imply that, for all $k\in \mathbb{N}$,
\begin{align*}
\mathbb{E} \left[ \left\| \bm{\mathsf{d}}_k \right\|_{\mathsf{H}_k}^2 \right] \leq 
\frac{\tilde{\beta}_{2k}^{\frac{1}{2}}}{\tilde{\beta}_{1k}^2 {v}_*^{\frac{1}{2}}} \left( \frac{\sigma^2}{b} + G^2 \right),
\end{align*}
completing the proof.
\end{proof}

\begin{lem}\label{lem:2}
Suppose that (S1)--(S3) and (A1)--(A2) hold. Then, Adam defined by Algorithm \ref{algo:1} satisfies the following: for all $k\in\mathbb{N}$ and all $\bm{\theta} \in \mathbb{R}^d$,
\begin{align*}
\mathbb{E}\left[ (\bm{\theta}_k - \bm{\theta})^\top \bm{m}_{k-1} \right]
\leq
\frac{D(\bm{\theta}) M^{\frac{1}{4}}}{{{v}_*^{\frac{1}{4}}} \beta_{1k}}
\sqrt{ \frac{\sigma^2}{b} + G^2}
+
\frac{\alpha_k \sqrt{\tilde{\beta}_{2k}}}{2 \sqrt{v_*} \beta_{1k} \tilde{\beta}_{1k}} \left( \frac{\sigma^2}{b} + G^2 \right)
+
D(\bm{\theta}) G \frac{\hat{\beta}_{1k}}{\beta_{1k}},
\end{align*}
where $\nabla f_{B_k}(\bm{\theta}_k) \odot \nabla f_{B_k}(\bm{\theta}_k) := (g_{k,i}^2) \in \mathbb{R}_{+}^d$, $M := \sup\{\max_{i\in [d]} g_{k,i}^2 \colon k\in \mathbb{N}\} < + \infty$, $\hat{\beta}_{1k} := 1 - \beta_{1k}$, $\tilde{\beta}_{1k} := 1 - \beta_{1k}^{k+1}$, $\tilde{\beta}_{2k} := 1 - \beta_{2k}^{k+1}$, ${v}_*$ is defined as in Lemma \ref{lem:bdd}, and $D(\bm{\theta})$ and $G$ are defined as in Assumptions (A1) and (A2).
\end{lem}

\begin{proof}
Let $\bm{\theta} \in \mathbb{R}^d$. Lemma \ref{lem:1} guarantees that for all $k\in \mathbb{N}$,
\begin{align}
\mathbb{E}\left[ (\bm{\theta}_k - \bm{\theta})^\top \bm{m}_{k-1} \right]
&= 
\underbrace{\frac{\tilde{\beta}_{1k}}{2 \alpha_k \beta_{1k}}
\left\{
\mathbb{E}\left[\left\| \bm{\theta}_{k} - \bm{\theta} \right\|_{\mathsf{H}_k}^2 \right]
- 
\mathbb{E}\left[ \left\| \bm{\theta}_{k+1} - \bm{\theta} \right\|_{\mathsf{H}_k}^2 \right]
\right\}}_{a_k}
+ 
\underbrace{\frac{\alpha_k \tilde{\beta}_{1k}}{2 \beta_{1k}} \mathbb{E}\left[\left\| \bm{\mathsf{d}}_k \right\|_{\mathsf{H}_k}^2 \right]}_{b_k} \nonumber \\
&\quad + 
\underbrace{\frac{\hat{\beta}_{1k}}{\beta_{1k}}
\mathbb{E}\left[ (\bm{\theta} - \bm{\theta}_k)^\top \nabla f (\bm{\theta}_k) \right]}_{c_k}.\label{key}
\end{align}
The triangle inequality and the definition of $\bm{\theta}_{k+1} := \bm{\theta}_{k} + \alpha_k \bm{\mathsf{d}}_k$ ensure that
\begin{align}\label{a_k}
\begin{split}
a_k 
&= \frac{\tilde{\beta}_{1k}}{2 \alpha_k \beta_{1k}}
\mathbb{E}\left[\left( 
\left\| \bm{\theta}_{k} - \bm{\theta} \right\|_{\mathsf{H}_k}
+ 
\left\| \bm{\theta}_{k+1} - \bm{\theta} \right\|_{\mathsf{H}_k} 
\right)
\left( 
\left\| \bm{\theta}_{k} - \bm{\theta} \right\|_{\mathsf{H}_k}
- 
\left\| \bm{\theta}_{k+1} - \bm{\theta} \right\|_{\mathsf{H}_k} 
\right)
\right]\\
&\leq
\frac{\tilde{\beta}_{1k}}{2 \alpha_k \beta_{1k}}
\mathbb{E}\left[\left( 
\left\| \bm{\theta}_{k} - \bm{\theta} \right\|_{\mathsf{H}_k}
+ 
\left\| \bm{\theta}_{k+1} - \bm{\theta} \right\|_{\mathsf{H}_k} 
\right) 
\left\| \bm{\theta}_{k} - \bm{\theta}_{k+1} \right\|_{\mathsf{H}_k} 
\right]\\
&= 
\frac{\tilde{\beta}_{1k}}{2\beta_{1k}}
\mathbb{E}\left[\left( 
\left\| \bm{\theta}_{k} - \bm{\theta} \right\|_{\mathsf{H}_k}
+ 
\left\| \bm{\theta}_{k+1} - \bm{\theta} \right\|_{\mathsf{H}_k} 
\right) 
\left\| \bm{\mathsf{d}}_k \right\|_{\mathsf{H}_k} 
\right].
\end{split}
\end{align}
Let $\nabla f_{B_k}(\bm{\theta}_k) \odot \nabla f_{B_k}(\bm{\theta}_k) := (g_{k,i}^2) \in \mathbb{R}_{+}^d$. Assumption (A1) ensures that there exists $M \in \mathbb{R}$ such that, for all $k\in \mathbb{N}$, $\max_{i\in [d]} g_{k,i}^2 \leq M$. The definition of ${\bm{v}}_k$ guarantees that, for all $i\in [d]$ and all $k\in \mathbb{N}$,
\begin{align*}
v_{k,i} = \beta_{2k} v_{k-1,i} + \hat{\beta}_{2k} g_{k,i}^2.
\end{align*}
Induction thus ensures that, for all $i\in [d]$ and all $k\in \mathbb{N}$,
\begin{align*}
v_{k,i} \leq \max \{ v_{0,i}, M \} = M,
\end{align*}
where $\bm{v}_0 = (v_{0,i}) = \bm{0}$. From the definition of $\hat{\bm{v}}_k$, we have that, for all $i\in [d]$ and all $k\in \mathbb{N}$,
\begin{align}\label{v_k}
\hat{v}_{k,i} = \frac{v_{k,i}}{\tilde{\beta}_{2k}}
\leq \frac{M}{\tilde{\beta}_{2k}},
\end{align}
which implies that 
\begin{align*}
\left\| \overline{\mathsf{H}}_k \right\| = \left\| \mathsf{diag}\left(\hat{v}_{k,i}^{\frac{1}{4}} \right) \right\| = {\max_{i\in [d]} \hat{v}_{k,i}^{\frac{1}{4}}} 
\leq 
\left(\frac{M}{\tilde{\beta}_{2k}} \right)^{\frac{1}{4}}.
\end{align*}
Hence, (A2) implies that, for all $k\in \mathbb{N}$,
\begin{align*}
&\left\| \bm{\theta}_{k} - \bm{\theta} \right\|_{\mathsf{H}_k}
= 
\left\| \overline{\mathsf{H}}_k (\bm{\theta}_{k} - \bm{\theta}) \right\|
\leq
\left\| \overline{\mathsf{H}}_k \right\| \left\|\bm{\theta}_{k} - \bm{\theta}\right\|
\leq
D(\bm{\theta}) \left(\frac{M}{\tilde{\beta}_{2k}} \right)^{\frac{1}{4}},\\
&\left\| \bm{\theta}_{k+1} - \bm{\theta} \right\|_{\mathsf{H}_k}
= 
\left\| \overline{\mathsf{H}}_k (\bm{\theta}_{k+1} - \bm{\theta}) \right\|
\leq
\left\| \overline{\mathsf{H}}_k \right\| \left\|\bm{\theta}_{k+1} - \bm{\theta}\right\|
\leq
D(\bm{\theta}) \left(\frac{M}{\tilde{\beta}_{2k}} \right)^{\frac{1}{4}}.
\end{align*} 
Lemma \ref{lem:bdd}, Jensen's inequality, and \eqref{a_k} ensure that, for all $k\in \mathbb{N}$, 
\begin{align}\label{A_k}
\begin{split}
a_k 
&\leq 
\frac{\tilde{\beta}_{1k}}{2\beta_{1k}}
2 D(\bm{\theta}) \left(\frac{M}{\tilde{\beta}_{2k}} \right)^{\frac{1}{4}}
\mathbb{E}\left[\left\| \bm{\mathsf{d}}_k \right\|_{\mathsf{H}_k} 
\right]
\leq
\frac{\tilde{\beta}_{1k}}{\beta_{1k}}
D(\bm{\theta}) \frac{M^{\frac{1}{4}}}{\tilde{\beta}_{2k}^{\frac{1}{4}}}
\frac{\tilde{\beta}_{2k}^{\frac{1}{4}}}{\tilde{\beta}_{1k} {v}_*^{\frac{1}{4}}} \sqrt{ \frac{\sigma^2}{b} + G^2}
\\
&= 
\frac{D(\bm{\theta}) M^{\frac{1}{4}}}{{{v}_*^{\frac{1}{4}}}\beta_{1k}}
\sqrt{ \frac{\sigma^2}{b} + G^2}.
\end{split}
\end{align}
Lemma \ref{lem:bdd} guarantees that, for all $k\in \mathbb{N}$, 
\begin{align}\label{B_k}
b_k = 
\frac{\alpha_k \tilde{\beta}_{1k}}{2 \beta_{1k}} \mathbb{E}\left[\left\| \bm{\mathsf{d}}_k \right\|_{\mathsf{H}_k}^2 \right]
\leq
\frac{\alpha_k \tilde{\beta}_{1k}}{2 \beta_{1k}}
\frac{\sqrt{\tilde{\beta}_{2k}}}{\tilde{\beta}_{1k}^2 \sqrt{{v}_*}} \left( \frac{\sigma^2}{b} + G^2 \right)
= 
\frac{\alpha_k \sqrt{\tilde{\beta}_{2k}}}{2 \sqrt{v_*} \beta_{1k} \tilde{\beta}_{1k}} \left( \frac{\sigma^2}{b} + G^2 \right).
\end{align}
The Cauchy--Schwarz inequality and Assumption (A2) imply that, for all $k\in \mathbb{N}$,
\begin{align}\label{C_k}
c_k = 
\frac{\hat{\beta}_{1k}}{\beta_{1k}}
\mathbb{E}\left[ (\bm{\theta} - \bm{\theta}_k)^\top \nabla f (\bm{\theta}_k) \right]
\leq
D(\bm{\theta}) G \frac{\hat{\beta}_{1k}}{\beta_{1k}}.
\end{align}
Therefore, \eqref{key}, \eqref{A_k}, \eqref{B_k}, and \eqref{C_k} ensure that, for all $k\in \mathbb{N}$,
\begin{align*}
\mathbb{E}\left[ (\bm{\theta}_k - \bm{\theta})^\top \bm{m}_{k-1} \right]
\leq
\frac{D(\bm{\theta}) M^{\frac{1}{4}}}{{{v}_*^{\frac{1}{4}}}\beta_{1k}}
\sqrt{ \frac{\sigma^2}{b} + G^2}
+
\frac{\alpha_k \sqrt{\tilde{\beta}_{2k}}}{2 \sqrt{v_*} \beta_{1k} \tilde{\beta}_{1k}} \left( \frac{\sigma^2}{b} + G^2 \right)
+
D(\bm{\theta}) G \frac{\hat{\beta}_{1k}}{\beta_{1k}},
\end{align*}
which completes the proof.
\end{proof}

\begin{lem}\label{lem:2_1_1}
Suppose that (S1)--(S3) and (A1)--(A2) hold. Then, Adam defined by Algorithm \ref{algo:1} satisfies the following: for all $k\in\mathbb{N}$ and all $\bm{\theta} \in \mathbb{R}^d$,
\begin{align*}
\mathbb{E}\left[ (\bm{\theta}_k - \bm{\theta})^\top \bm{m}_{k} \right]
&\leq
\frac{D(\bm{\theta}) M^{\frac{1}{4}}}{{{v}_*^{\frac{1}{4}}} \beta_{1k}}
\sqrt{ \frac{\sigma^2}{b} + G^2}
+
\frac{\alpha_k \sqrt{\tilde{\beta}_{2k}}}{2 \sqrt{v_*} \beta_{1k} \tilde{\beta}_{1k}} \left( \frac{\sigma^2}{b} + G^2 \right)
+
D(\bm{\theta}) G \frac{\hat{\beta}_{1k}}{\beta_{1k}}\\
&\quad + \hat{\beta}_{1k} D(\bm{\theta}) \left( B + \sqrt{\frac{\sigma^2}{b} + G^2} \right),
\end{align*}
where the parameters are defined as in Lemma \ref{lem:2}.
\end{lem}

\begin{proof}
Let $\bm{\theta} \in \mathbb{R}^d$ and $k\in \mathbb{N}$. The definition of $\bm{m}_k$ implies that
\begin{align*}
(\bm{\theta}_k - \bm{\theta})^\top \bm{m}_{k}
&= 
(\bm{\theta}_k - \bm{\theta})^\top \bm{m}_{k-1}
+ 
(\bm{\theta}_k - \bm{\theta})^\top (\bm{m}_{k} - \bm{m}_{k-1})\\
&=
(\bm{\theta}_k - \bm{\theta})^\top \bm{m}_{k-1}
+
\hat{\beta}_{1k} (\bm{\theta}_k - \bm{\theta})^\top (\nabla f_{B_k}(\bm{\theta}_k) - \bm{m}_{k-1}),
\end{align*}
which, together with the Cauchy--Schwarz inequality, the triangle inequality, and Assumptions (A1) and (A2), implies that
\begin{align*}
(\bm{\theta}_k - \bm{\theta})^\top \bm{m}_{k}
&\leq
(\bm{\theta}_k - \bm{\theta})^\top \bm{m}_{k-1}
+
\hat{\beta}_{1k} D(\bm{\theta}) \|\nabla f_{B_k}(\bm{\theta}_k) - \bm{m}_{k-1}\|\\
&\leq 
(\bm{\theta}_k - \bm{\theta})^\top \bm{m}_{k-1}
+
\hat{\beta}_{1k} D(\bm{\theta}) (B + \|\bm{m}_{k-1}\|).
\end{align*}
Lemma \ref{lem:bdd} and Jensen's inequality guarantee that
\begin{align}\label{ineq_1}
\mathbb{E} \left[
(\bm{\theta}_k - \bm{\theta})^\top \bm{m}_{k}
\right]
\leq
\mathbb{E} \left[
(\bm{\theta}_k - \bm{\theta})^\top \bm{m}_{k-1}
\right]
+ 
\hat{\beta}_{1k}
D(\bm{\theta}) \left(B + \sqrt{\frac{\sigma^2}{b} + G^2} \right).
\end{align}
Hence, Lemma \ref{lem:2} implies that 
\begin{align*}
\mathbb{E}\left[ (\bm{\theta}_k - \bm{\theta})^\top \bm{m}_{k} \right]
&\leq
\frac{D(\bm{\theta}) M^{\frac{1}{4}}}{{{v}_*^{\frac{1}{4}}} \beta_{1k}}
\sqrt{ \frac{\sigma^2}{b} + G^2}
+
\frac{\alpha_k \sqrt{\tilde{\beta}_{2k}}}{2 \sqrt{v_*} \beta_{1k} \tilde{\beta}_{1k}} \left( \frac{\sigma^2}{b} + G^2 \right)
+
D(\bm{\theta}) G \frac{\hat{\beta}_{1k}}{\beta_{1k}}\\
&\quad + \hat{\beta}_{1k} D(\bm{\theta}) \left( B + \sqrt{\frac{\sigma^2}{b} + G^2} \right),
\end{align*}
which completes the proof.
\end{proof}

\begin{lem}\label{lem:3}
Suppose that (S1)--(S3) and (A1)--(A2) hold. Then, Adam defined by Algorithm \ref{algo:1} satisfies the following: for all $k\in\mathbb{N}$ and all $\bm{\theta} \in \mathbb{R}^d$,
\begin{align*}
\mathbb{E}\left[
(\bm{\theta}_k - \bm{\theta})^\top \nabla f(\bm{\theta}_k)
\right]
&\leq
\frac{D(\bm{\theta}) M^{\frac{1}{4}}}{{{v}_*^{\frac{1}{4}}} \beta_{1k}}
\sqrt{ \frac{\sigma^2}{b} + G^2}
+
\frac{\alpha_k \sqrt{\tilde{\beta}_{2k}}}{2 \sqrt{v_*} \beta_{1k} \tilde{\beta}_{1k}} \left( \frac{\sigma^2}{b} + G^2 \right)
+
D(\bm{\theta}) G \frac{\hat{\beta}_{1k}}{\beta_{1k}}\\
&\quad + 
D(\bm{\theta})\left(\frac{1}{\beta_{1k}} + 2 \hat{\beta}_{1k} \right)
\left(B + \sqrt{\frac{\sigma^2}{b} + G^2} \right),
\end{align*}
where the parameters are defined as in Lemma \ref{lem:2}.
\end{lem}

\begin{proof}
Let $\bm{\theta} \in \mathbb{R}^d$ and $k\in\mathbb{N}$. The definition of $\bm{m}_k$ ensures that
\begin{align*}
&(\bm{\theta}_k - \bm{\theta})^\top \nabla f_{B_k}(\bm{\theta}_k)\\
&=
(\bm{\theta}_k - \bm{\theta})^\top \bm{m}_k
+ 
(\bm{\theta}_k - \bm{\theta})^\top (\nabla f_{B_k}(\bm{\theta}_k) - \bm{m}_{k-1}) 
+
(\bm{\theta}_k - \bm{\theta})^\top (\bm{m}_{k-1} - \bm{m}_{k})\\
&=
(\bm{\theta}_k - \bm{\theta})^\top \bm{m}_k
+ 
\frac{1}{\beta_{1k}}(\bm{\theta}_k - \bm{\theta})^\top (\nabla f_{B_k}(\bm{\theta}_k) - \bm{m}_{k})
+
\hat{\beta}_{1k} (\bm{\theta}_k - \bm{\theta})^\top (\bm{m}_{k-1} - \nabla f_{B_k}(\bm{\theta}_k)),
\end{align*}
which, together with the Cauchy--Schwarz inequality, the triangle inequality, and Assumptions (A1) and (A2), implies that
\begin{align*}
(\bm{\theta}_k - \bm{\theta})^\top \nabla f_{B_k}(\bm{\theta}_k)
\leq 
(\bm{\theta}_k - \bm{\theta})^\top \bm{m}_k
+ 
\frac{1}{\beta_{1k}} D(\bm{\theta}) (B + \|\bm{m}_{k}\|)
+ 
\hat{\beta}_{1k} D(\bm{\theta}) (B + \|\bm{m}_{k-1}\|).
\end{align*}
Lemma \ref{lem:bdd} and Jensen's inequality guarantee that
\begin{align}\label{ineq_2}
\mathbb{E}\left[
(\bm{\theta}_k - \bm{\theta})^\top \nabla f(\bm{\theta}_k)
\right]
\leq
\mathbb{E}\left[
(\bm{\theta}_k - \bm{\theta})^\top \bm{m}_k
\right]
+ 
\left(\frac{1}{\beta_{1k}} + \hat{\beta}_{1k} \right)
D(\bm{\theta}) \left(B + \sqrt{\frac{\sigma^2}{b} + G^2} \right),
\end{align}
which, together with Lemma \ref{lem:2_1_1}, implies that
\begin{align*}
\mathbb{E}\left[
(\bm{\theta}_k - \bm{\theta})^\top \nabla f(\bm{\theta}_k)
\right]
&\leq
\frac{D(\bm{\theta}) M^{\frac{1}{4}}}{{{v}_*^{\frac{1}{4}}} \beta_{1k}}
\sqrt{ \frac{\sigma^2}{b} + G^2}
+
\frac{\alpha_k \sqrt{\tilde{\beta}_{2k}}}{2 \sqrt{v_*} \beta_{1k} \tilde{\beta}_{1k}} \left( \frac{\sigma^2}{b} + G^2 \right)
+
D(\bm{\theta}) G \frac{\hat{\beta}_{1k}}{\beta_{1k}}\\
&\quad + 
D(\bm{\theta})\left(\frac{1}{\beta_{1k}} + 2 \hat{\beta}_{1k} \right)
\left(B + \sqrt{\frac{\sigma^2}{b} + G^2} \right),
\end{align*}
which completes the proof.
\end{proof}

\begin{lem}\label{lem:2_1}
Suppose that (S1)--(S3) and (A1)--(A2) hold, $\beta_{1k} := \beta_1 \in (0,1)$, and $(\alpha_k)_{k\in\mathbb{N}}$ is monotone decreasing. Then, Adam defined by Algorithm \ref{algo:1} with \eqref{max_1} satisfies the following: for all $K\geq 1$ and all $\bm{\theta} \in \mathbb{R}^d$,
\begin{align*}
\frac{1}{K}\sum_{k=1}^K \mathbb{E}\left[ (\bm{\theta}_k - \bm{\theta})^\top \bm{m}_{k-1} \right]
\leq
\frac{d \tilde{D}(\bm{\theta}) \sqrt{M} \tilde{\beta}_{1K}}{2 \beta_1 \alpha_K \sqrt{\tilde{\beta}_{2K}}K}
+
\frac{(\sigma^2 b^{-1} + G^2)}{2 \sqrt{v_*} \beta_{1} \hat{\beta}_1 K}
\sum_{k=1}^K \alpha_k \sqrt{\tilde{\beta}_{2k}} 
+
D(\bm{\theta}) G \frac{\hat{\beta}_{1}}{\beta_{1}},
\end{align*}
where the parameters are defined as in Lemma \ref{lem:2} and $\tilde{D} (\bm{\theta}) := \sup \{ \max_{i\in [d]} (\theta_{k,i} - \theta_i)^2 \colon k \in \mathbb{N} \} < + \infty$.
\end{lem}

\begin{proof}
Let $\bm{\theta} \in \mathbb{R}^d$ and 
\begin{align*}
\gamma_k := \frac{\tilde{\beta}_{1k}}{2 \beta_{1} \alpha_k}
\end{align*}
for all $k\in \mathbb{N}$. Since $(\alpha_k)_{k\in\mathbb{N}}$ is monotone decreasing and $\tilde{\beta}_{1k} = 1 - \beta_1^{k+1} \leq 1 - \beta_1^{k+2} = \tilde{\beta}_{1,k+1}$, $(\gamma_k)_{k\in\mathbb{N}}$ is monotone increasing. From the definition of $a_k$ in \eqref{key}, we have that, for all $K \geq 1$,
\begin{align}\label{LAM}
\begin{split}
\sum_{k = 1}^K a_k
&=
\gamma_1 \mathbb{E}\left[ \left\| \bm{\theta}_{1} - \bm{\theta} \right\|_{\mathsf{H}_{1}}^2\right]
+
\underbrace{
\sum_{k=2}^K \left\{
\gamma_k \mathbb{E}\left[ \left\| \bm{\theta}_{k} - \bm{\theta} \right\|_{\mathsf{H}_{k}}^2\right]
-
\gamma_{k-1} \mathbb{E}\left[ \left\| \bm{\theta}_{k} - \bm{\theta} \right\|_{\mathsf{H}_{k-1}}^2\right] 
\right\}
}_{{\Gamma}_K}\\
&\quad-
\gamma_{K} \mathbb{E} \left[ \left\| \bm{\theta}_{K+1} - \bm{\theta} \right\|_{\mathsf{H}_{K}}^2 \right].
\end{split}
\end{align}
Since $\overline{\mathsf{H}}_k \in \mathbb{S}_{++}^d$ exists such that $\mathsf{H}_k = \overline{\mathsf{H}}_k^2$, we have $\|\bm{x}\|_{\mathsf{H}_k}^2 = \| \overline{\mathsf{H}}_k \bm{x} \|^2$ for all $\bm{x}\in\mathbb{R}^d$. Accordingly, we have 
\begin{align*}
{\Gamma}_K 
=
\mathbb{E} \left[ 
\sum_{k=2}^K 
\left\{
\gamma_{k} \left\| \overline{\mathsf{H}}_{k} (\bm{\theta}_{k} - \bm{\theta}) \right\|^2
-
\gamma_{k-1} \left\| \overline{\mathsf{H}}_{k-1} (\bm{\theta}_{k} - \bm{\theta}) \right\|^2
\right\}
\right].
\end{align*}
From $\overline{\mathsf{H}}_{k} = \mathsf{diag}(\hat{v}_{k,i}^{1/4})$, we have that, for all $\bm{x} = (x_i)_{i=1}^d \in \mathbb{R}^d$, $\| \overline{\mathsf{H}}_{k} \bm{x} \|^2 = \sum_{i=1}^d \sqrt{\hat{v}_{k,i}} x_i^2$. Hence, for all $K\geq 2$,
\begin{align}\label{DELTA}
{\Gamma}_K 
= 
\mathbb{E} \left[ 
\sum_{k=2}^K
\sum_{i=1}^d 
\left(
\gamma_{k} \sqrt{\hat{v}_{k,i}}
-
\gamma_{k-1} \sqrt{\hat{v}_{k-1,i}}
\right)
(\theta_{k,i} - \theta_i)^2
\right].
\end{align}
Condition \eqref{max_1} and $\gamma_k \geq \gamma_{k-1}$ ($k \geq 1$) imply that, for all $k \geq 1$ and all $i\in [d]$,
\begin{align*}
\gamma_{k} \sqrt{\hat{v}_{k,i}} - \gamma_{k-1} \sqrt{\hat{v}_{k-1,i}} \geq 0.
\end{align*} 
Moreover, (A2) ensures that $\tilde{D} (\bm{\theta}) := \sup \{ \max_{i\in [d]} (\theta_{k,i} - \theta_i)^2 \colon k \in \mathbb{N} \} < + \infty$. Accordingly, for all $K \geq 2$,
\begin{align*}
{\Gamma}_K
\leq
\tilde{D}(\bm{\theta})
\mathbb{E} \left[ 
\sum_{k=2}^K
\sum_{i=1}^d 
\left(
\gamma_{k}\sqrt{\hat{v}_{k,i}} - \gamma_{k-1} \sqrt{\hat{v}_{k-1,i}}
\right)
\right]
= 
\tilde{D}(\bm{\theta})
\mathbb{E} \left[ 
\sum_{i=1}^d
\left(
\gamma_{K} \sqrt{\hat{v}_{K,i}}
-
\gamma_{1} \sqrt{\hat{v}_{1,i}}
\right)
\right].
\end{align*}
Therefore, (\ref{LAM}), $\mathbb{E} [\| \bm{\theta}_{1} - \bm{\theta}\|_{\mathsf{H}_{1}}^2] \leq \tilde{D}(\bm{\theta}) \mathbb{E} [ \sum_{i=1}^d \sqrt{\hat{v}_{1,i}}]$, and \eqref{v_k} imply, for all $K\geq 1$,
\begin{align}\label{a_K_1}
\begin{split}
\sum_{k=1}^K a_k
&\leq
\gamma_{1} \tilde{D}(\bm{\theta}) \mathbb{E} \left[ 
\sum_{i=1}^d \sqrt{\hat{v}_{1,i}} \right]
+
\tilde{D}(\bm{\theta})
\mathbb{E} \left[
\sum_{i=1}^d 
\left(
\gamma_{K} \sqrt{\hat{v}_{K,i}}
-
\gamma_{1} \sqrt{\hat{v}_{1,i}}
\right)
\right]\\
&=
\gamma_{K} \tilde{D}(\bm{\theta})
\mathbb{E} \left[
\sum_{i=1}^d 
\sqrt{\hat{v}_{K,i}}
\right]\\
&\leq 
{\gamma}_K \tilde{D}(\bm{\theta}) 
\sum_{i=1}^d 
\sqrt{\frac{M}{\tilde{\beta}_{2K}}}\\
&\leq 
\frac{d \tilde{D}(\bm{\theta}) \sqrt{M} \tilde{\beta}_{1K}}{2 \beta_1 \alpha_K \sqrt{\tilde{\beta}_{2K}}}.
\end{split}
\end{align}
Inequality \eqref{B_k} with $\beta_{1k} = \beta_1$ and $\tilde{\beta}_{1k} := 1 - \beta_1^{k+1} \geq 1 - \beta_1 =: \hat{\beta}_1$ implies that 
\begin{align}\label{B_k_1}
b_k 
\leq
\frac{\alpha_k \sqrt{\tilde{\beta}_{2k}}}{2 \sqrt{v_*} \beta_{1k} \tilde{\beta}_{1k}} \left( \frac{\sigma^2}{b} + G^2 \right)
\leq 
\frac{\alpha_k \sqrt{\tilde{\beta}_{2k}}}{2 \sqrt{v_*} \beta_{1} \hat{\beta}_1} \left( \frac{\sigma^2}{b} + G^2 \right).
\end{align}
Inequality \eqref{C_k} with $\beta_{1k} = \beta_1$ implies that 
\begin{align}\label{C_k_1}
c_k 
\leq
D(\bm{\theta}) G \frac{\hat{\beta}_{1k}}{\beta_{1k}}
= 
D(\bm{\theta}) G \frac{\hat{\beta}_{1}}{\beta_{1}}.
\end{align}
Hence, \eqref{key}, \eqref{a_K_1}, \eqref{B_k_1}, and \eqref{C_k_1} ensure that, for all $K\geq 1$,
\begin{align*}
\frac{1}{K}\sum_{k=1}^K \mathbb{E}\left[ (\bm{\theta}_k - \bm{\theta})^\top \bm{m}_{k-1} \right]
\leq
\frac{d \tilde{D}(\bm{\theta}) \sqrt{M} \tilde{\beta}_{1K}}{2 \beta_1 \alpha_K \sqrt{\tilde{\beta}_{2K}}K}
+
\frac{(\sigma^2 b^{-1} + G^2)}{2 \sqrt{v_*} \beta_{1} \hat{\beta}_1 K}
\sum_{k=1}^K \alpha_k \sqrt{\tilde{\beta}_{2k}} 
+
D(\bm{\theta}) G \frac{\hat{\beta}_{1}}{\beta_{1}},
\end{align*}
which completes the proof.
\end{proof}

\begin{lem}\label{lem:3_1}
Suppose that (S1)--(S3) and (A1)--(A2) hold, $\beta_{1k} := \beta_1 \in (0,1)$, and $(\alpha_k)_{k\in\mathbb{N}}$ is monotone decreasing. Then, Adam defined by Algorithm \ref{algo:1} with \eqref{max_1} satisfies the following: for all $K\geq 1$ and all $\bm{\theta} \in \mathbb{R}^d$,
\begin{align*}
\frac{1}{K}\sum_{k=1}^K \mathbb{E}\left[ (\bm{\theta}_k - \bm{\theta})^\top \bm{m}_{k} \right]
&\leq
\frac{d \tilde{D}(\bm{\theta}) \sqrt{M} \tilde{\beta}_{1K}}{2 \beta_1 \alpha_K \sqrt{\tilde{\beta}_{2K}}K}
+
\frac{(\sigma^2 b^{-1} + G^2)}{2 \sqrt{v_*} \beta_{1} \hat{\beta}_1 K}
\sum_{k=1}^K \alpha_k \sqrt{\tilde{\beta}_{2k}} 
+
D(\bm{\theta}) G \frac{\hat{\beta}_{1}}{\beta_{1}}\\
&\quad + 
\hat{\beta}_{1}
D(\bm{\theta}) \left(B + \sqrt{\frac{\sigma^2}{b} + G^2} \right),
\end{align*}
where the parameters are defined as in Lemma \ref{lem:2_1}.
\end{lem}

\begin{proof}
Let $\bm{\theta} \in \mathbb{R}^d$. Inequality \eqref{ineq_1} with $\beta_{1k} = \beta_1$ implies that, for all $K \geq 1$,
\begin{align*}
\frac{1}{K} \sum_{k=1}^K \mathbb{E} \left[
(\bm{\theta}_k - \bm{\theta})^\top \bm{m}_{k}
\right]
\leq
\frac{1}{K} \sum_{k=1}^K \mathbb{E} \left[
(\bm{\theta}_k - \bm{\theta})^\top \bm{m}_{k-1}
\right]
+ 
\hat{\beta}_{1}
D(\bm{\theta}) \left(B + \sqrt{\frac{\sigma^2}{b} + G^2} \right).
\end{align*}
Hence, Lemma \ref{lem:2_1} leads to Lemma \ref{lem:3_1}.
\end{proof}

\begin{lem}\label{lem:4_1}
Suppose that (S1)--(S3) and (A1)--(A2) hold, $\beta_{1k} := \beta_1 \in (0,1)$, and $(\alpha_k)_{k\in\mathbb{N}}$ is monotone decreasing. Then, Adam defined by Algorithm \ref{algo:1} with \eqref{max_1} satisfies the following: for all $K\geq 1$ and all $\bm{\theta} \in \mathbb{R}^d$,
\begin{align*}
\frac{1}{K}\sum_{k=1}^K \mathbb{E}\left[ (\bm{\theta}_k - \bm{\theta})^\top \nabla f (\bm{\theta}_{k}) \right]
&\leq
\frac{d \tilde{D}(\bm{\theta}) \sqrt{M} \tilde{\beta}_{1K}}{2 \beta_1 \alpha_K \sqrt{\tilde{\beta}_{2K}}K}
+
\frac{(\sigma^2 b^{-1} + G^2)}{2 \sqrt{v_*} \beta_{1} \hat{\beta}_1 K}
\sum_{k=1}^K \alpha_k \sqrt{\tilde{\beta}_{2k}} 
+
D(\bm{\theta}) G \frac{\hat{\beta}_{1}}{\beta_{1}}\\
&\quad + 
\left( \frac{1}{\beta_1} + 2\hat{\beta}_{1} \right)
D(\bm{\theta}) \left(B + \sqrt{\frac{\sigma^2}{b} + G^2} \right),
\end{align*}
where the parameters are defined as in Lemma \ref{lem:2_1}.
\end{lem}

\begin{proof}
Let $\bm{\theta} \in \mathbb{R}^d$. Inequality \eqref{ineq_2} with $\beta_{1k} = \beta_1$ implies that, for all $K \geq 1$,
\begin{align*}
&\frac{1}{K} \sum_{k=1}^K \mathbb{E}\left[
(\bm{\theta}_k - \bm{\theta})^\top \nabla f(\bm{\theta}_k)
\right]\\
&\leq
\frac{1}{K} \sum_{k=1}^K \mathbb{E}\left[
(\bm{\theta}_k - \bm{\theta})^\top \bm{m}_k
\right]
+ 
\left(\frac{1}{\beta_{1}} + \hat{\beta}_{1} \right)
D(\bm{\theta}) \left(B + \sqrt{\frac{\sigma^2}{b} + G^2} \right),
\end{align*}
which, together with Lemma \ref{lem:3_1}, shows that Lemma \ref{lem:4_1} holds.
\end{proof}

\subsection{Proof of Theorem \ref{c1}}
\begin{proof}
Lemmas \ref{lem:2_1_1} and \ref{lem:3} with 
\begin{align*}
\alpha_k = \alpha, \text{ }
\beta_{1k} = \beta_1, \text{ }
\beta_{2k} = \beta_2, \text{ }
\tilde{\beta}_{1k} = 1 - \beta_1^{k+1},\text{ }
\tilde{\beta}_{2k} = 1 - \beta_2^{k+1}, \text{ }
\hat{\beta}_{1} = 1 - \beta_1
\end{align*}
imply that
\begin{align*}
\mathbb{E}\left[ (\bm{\theta}_k - \bm{\theta}^\star)^\top \bm{m}_{k} \right]
&\leq
\frac{D(\bm{\theta}^\star) M^{\frac{1}{4}}}{{{v}_*^{\frac{1}{4}}} \beta_{1}}
\sqrt{ \frac{\sigma^2}{b} + G^2}
+
\frac{\alpha \sqrt{\tilde{\beta}_{2k}}}{2 \sqrt{v_*} \beta_{1} \tilde{\beta}_{1k}} \left( \frac{\sigma^2}{b} + G^2 \right)
+
D(\bm{\theta}^\star) G \frac{\hat{\beta}_{1}}{\beta_{1}}\\
&\quad + \hat{\beta}_{1} D(\bm{\theta}^\star) \left( B + \sqrt{\frac{\sigma^2}{b} + G^2} \right),\\
\mathbb{E}\left[
(\bm{\theta}_k - \bm{\theta})^\top \nabla L(\bm{\theta}_k)
\right]
&\leq
\frac{D(\bm{\theta}) M^{\frac{1}{4}}}{{{v}_*^{\frac{1}{4}}} \beta_{1}}
\sqrt{ \frac{\sigma^2}{b} + G^2}
+
\frac{\alpha \sqrt{\tilde{\beta}_{2k}}}{2 \sqrt{v_*} \beta_{1} \tilde{\beta}_{1k}} \left( \frac{\sigma^2}{b} + G^2 \right)\\
&\quad +
D(\bm{\theta}) G \frac{\hat{\beta}_{1}}{\beta_{1}}
+ 
D(\bm{\theta})\left(\frac{1}{\beta_{1}} + 2 \hat{\beta}_{1} \right)
\left(B + \sqrt{\frac{\sigma^2}{b} + G^2} \right),
\end{align*}
which completes the proof.
\end{proof}

\subsection{Proof of Corollary \ref{cor:1}}
\begin{proof}
The sequences $(\tilde{\beta}_{1k})_{k\in\mathbb{N}}$ and $(\tilde{\beta}_{2k})_{k\in\mathbb{N}}$ converge to $1$. Theorem \ref{c1} thus leads to Corollary \ref{cor:1}.
\end{proof}

\subsection{Proof of Theorem \ref{d1}}
\begin{proof}
Lemmas \ref{lem:2_1_1} and \ref{lem:3} with 
\begin{align*}
&\alpha_k = \frac{1}{k^a}, \text{ }
\beta_{1k} = 1 - \frac{1}{k^{b_1}}, \text{ }
\beta_{2k} = \left( 1 - \frac{1}{k^{b_2}} \right)^{\frac{1}{k+1}}, \text{ }
\tilde{\beta}_{1k} = 1 - \beta_{1k}^{k+1} \geq 1 - \beta_{1k},\text{ }
\tilde{\beta}_{2k} = 1 - \beta_{2k}^{k+1},\\
&\hat{\beta}_{1k} = 1 - \beta_{1k}
\end{align*}
imply that
\begin{align*}
\mathbb{E}\left[ (\bm{\theta}_k - \bm{\theta})^\top \bm{m}_{k} \right]
&\leq
\frac{D(\bm{\theta}) M^{\frac{1}{4}}}{{{v}_*^{\frac{1}{4}}} \beta_{1k}}
\sqrt{ \frac{\sigma^2}{b} + G^2}
+
\frac{\alpha_k \sqrt{\tilde{\beta}_{2k}}}{2 \sqrt{v_*} \beta_{1k} \tilde{\beta}_{1k}} \left( \frac{\sigma^2}{b} + G^2 \right)
+
D(\bm{\theta}) G \frac{\hat{\beta}_{1k}}{\beta_{1k}}\\
&\quad + \hat{\beta}_{1k} D(\bm{\theta}) \left( B + \sqrt{\frac{\sigma^2}{b} + G^2} \right)\\
&\leq
\frac{D(\bm{\theta}^\star) M^{\frac{1}{4}} k^{b_1}}{v_*^{\frac{1}{4}}(k^{b_1} - 1)}\sqrt{ \frac{\sigma^2}{b} + G^2}
+
\frac{1}{2 \sqrt{v_*} (k^{b_1} - 1) k^{a +\frac{b_2}{2} -2 b_1}} \left( \frac{\sigma^2}{b} + G^2 \right)\\
&\quad + \frac{1}{k^{b_1} -1} D(\bm{\theta}^\star) G
+ \frac{1}{k^{b_1}} D(\bm{\theta}^\star) \left( B + \sqrt{ \frac{\sigma^2}{b} + G^2} \right),
\end{align*}
\begin{align*}
\mathbb{E}\left[
(\bm{\theta}_k - \bm{\theta})^\top \nabla f(\bm{\theta}_k)
\right]
&\leq
\frac{D(\bm{\theta}) M^{\frac{1}{4}}}{{{v}_*^{\frac{1}{4}}} \beta_{1k}}
\sqrt{ \frac{\sigma^2}{b} + G^2}
+
\frac{\alpha_k \sqrt{\tilde{\beta}_{2k}}}{2 \sqrt{v_*} \beta_{1k} \tilde{\beta}_{1k}} \left( \frac{\sigma^2}{b} + G^2 \right)
+
D(\bm{\theta}) G \frac{\hat{\beta}_{1k}}{\beta_{1k}}\\
&\quad + 
D(\bm{\theta})\left(\frac{1}{\beta_{1k}} + 2 \hat{\beta}_{1k} \right)
\left(B + \sqrt{\frac{\sigma^2}{b} + G^2} \right)\\
&\leq
\frac{D(\bm{\theta}) M^{\frac{1}{4}} k^{b_1}}{v_*^{\frac{1}{4}}(k^{b_1} - 1)}\sqrt{ \frac{\sigma^2}{b} + G^2}
+
\frac{1}{2 \sqrt{v_*} (k^{b_1} - 1) k^{a +\frac{b_2}{2} -2 b_1}} \left( \frac{\sigma^2}{b} + G^2 \right)\\
&\quad + \frac{1}{k^{b_1} -1} D(\bm{\theta}) G 
+ \frac{1}{k^{b_1}} D(\bm{\theta}) \left( B + \sqrt{ \frac{\sigma^2}{b} + G^2} \right)\\
&\quad +
\frac{k^{{b_1}^2} + 2 k^{b_1} - 2}{k^{b_1} (k^{b_1} - 1)}
D(\bm{\theta}) \left(B + \sqrt{\frac{\sigma^2}{b} + G^2} \right),\end{align*}
which completes the proof.
\end{proof}

\subsection{Proof of Corollary \ref{cor:2}}
\begin{proof}
Since $a - b_1 + b_2/2 > 0$, we have that 
\begin{align*}
\frac{1}{(k^{b_1} - 1) k^{a +\frac{b_2}{2} -2 b_1}}
= 
\frac{1}{k^{a +\frac{b_2}{2} - b_1}- k^{a +\frac{b_2}{2} -2 b_1}}
=
\frac{1}{k^{a +\frac{b_2}{2} - b_1}}\left({1 - \frac{1}{k^{b_1}}}\right)^{-1} \to 0.
\end{align*}
Theorem \ref{d1} thus leads to Corollary \ref{cor:2}.
\end{proof}

\subsection{Proof of Theorem \ref{c2}}
\begin{proof}
Lemmas \ref{lem:3_1} and \ref{lem:4_1} with 
\begin{align*}
\alpha_k = \alpha, \text{ }
\beta_{1k} = \beta_1, \text{ }
\beta_{2k} = \beta_2, \text{ }
\tilde{\beta}_{1k} = 1 - \beta_1^{k+1},\text{ }
\tilde{\beta}_{2k} = 1 - \beta_2^{k+1} \leq 1, \text{ }
\hat{\beta}_{1} = 1 - \beta_1
\end{align*}
ensure that
\begin{align*}
\frac{1}{K}\sum_{k=1}^K \mathbb{E}\left[ (\bm{\theta}_k - \bm{\theta}^\star)^\top \bm{m}_{k} \right]
&\leq
\frac{d \tilde{D}(\bm{\theta}^\star) \sqrt{M} \tilde{\beta}_{1K}}{2 \beta_1 \alpha \sqrt{\tilde{\beta}_{2K}}K}
+
\frac{(\sigma^2 b^{-1} + G^2)}{2 \sqrt{v_*} \beta_{1} \hat{\beta}_1 K}
\sum_{k=1}^K \alpha \sqrt{\tilde{\beta}_{2k}} 
+
D(\bm{\theta}^\star) G \frac{\hat{\beta}_{1}}{\beta_{1}}\\
&\quad + 
\hat{\beta}_{1}
D(\bm{\theta}^\star) \left(B + \sqrt{\frac{\sigma^2}{b} + G^2} \right)\\
&\leq
\frac{d \tilde{D}(\bm{\theta}^\star) \sqrt{M} \tilde{\beta}_{1K}}{2 \beta_1 \alpha \sqrt{\tilde{\beta}_{2K}}K}
+
\frac{(\sigma^2 b^{-1} + G^2)}{2 \sqrt{v_*} \beta_{1} \hat{\beta}_1}
\alpha 
+
D(\bm{\theta}^\star) G \frac{\hat{\beta}_{1}}{\beta_{1}}\\
&\quad + 
\hat{\beta}_{1}
D(\bm{\theta}^\star) \left(B + \sqrt{\frac{\sigma^2}{b} + G^2} \right)
\end{align*}
and that
\begin{align*}
\frac{1}{K}\sum_{k=1}^K \mathbb{E}\left[ (\bm{\theta}_k - \bm{\theta})^\top \nabla f (\bm{\theta}_{k}) \right]
&\leq
\frac{d \tilde{D}(\bm{\theta}) \sqrt{M} \tilde{\beta}_{1K}}{2 \beta_1 \alpha \sqrt{\tilde{\beta}_{2K}}K}
+
\frac{(\sigma^2 b^{-1} + G^2)}{2 \sqrt{v_*} \beta_{1} \hat{\beta}_1 K}
\sum_{k=1}^K \alpha \sqrt{\tilde{\beta}_{2k}} 
+
D(\bm{\theta}) G \frac{\hat{\beta}_{1}}{\beta_{1}}\\
&\quad + 
\left( \frac{1}{\beta_1} + 2\hat{\beta}_{1} \right)
D(\bm{\theta}) \left(B + \sqrt{\frac{\sigma^2}{b} + G^2} \right)\\
&\leq
\frac{d \tilde{D}(\bm{\theta}) \sqrt{M} \tilde{\beta}_{1K}}{2 \beta_1 \alpha \sqrt{\tilde{\beta}_{2K}}K}
+
\frac{(\sigma^2 b^{-1} + G^2)}{2 \sqrt{v_*} \beta_{1} \hat{\beta}_1}
\alpha 
+
D(\bm{\theta}) G \frac{\hat{\beta}_{1}}{\beta_{1}}\\
&\quad + 
\left( \frac{1}{\beta_1} + 2\hat{\beta}_{1} \right)
D(\bm{\theta}) \left(B + \sqrt{\frac{\sigma^2}{b} + G^2} \right), 
\end{align*}
which completes the proof.
\end{proof}

\subsection{Proof of Theorem \ref{c3}}
\begin{proof}
Let
\begin{align*}
\alpha_k = \alpha, \text{ }
\beta_{1k} = \beta_1, \text{ }
\beta_{2k} = \left( 1 - \frac{1}{k^{b_2}} \right)^{\frac{1}{k+1}}, \text{ }
\tilde{\beta}_{1k} = 1 - \beta_1^{k+1} \leq 1,\text{ }
\tilde{\beta}_{2k} = 1 - \beta_{2k}^{k+1}, \text{ }
\hat{\beta}_{1} = 1 - \beta_1,
\end{align*}
where $b_2 \in (0,2)$. We have that 
\begin{align*}
\sqrt{\tilde{\beta}_{2k}} = \sqrt{1 - \beta_{2k}^{k+1}}
= \sqrt{\frac{1}{k^{b_2}}}.
\end{align*}
Lemma \ref{lem:3_1} ensures that
\begin{align*}
\frac{1}{K}\sum_{k=1}^K \mathbb{E}\left[ (\bm{\theta}_k - \bm{\theta}^\star)^\top \bm{m}_{k} \right]
&\leq
\frac{d \tilde{D}(\bm{\theta}^\star) \sqrt{M} \tilde{\beta}_{1K}}{2 \beta_1 \alpha \sqrt{\tilde{\beta}_{2K}}K}
+
\frac{(\sigma^2 b^{-1} + G^2)}{2 \sqrt{v_*} \beta_{1} \hat{\beta}_1 K}
\sum_{k=1}^K \alpha \sqrt{\tilde{\beta}_{2k}} 
+
D(\bm{\theta}^\star) G \frac{\hat{\beta}_{1}}{\beta_{1}}\\
&\quad + 
\hat{\beta}_{1}
D(\bm{\theta}^\star) \left(B + \sqrt{\frac{\sigma^2}{b} + G^2} \right)\\
&\leq
\frac{d \tilde{D}(\bm{\theta}^\star) \sqrt{M}}{2 \beta_1 \alpha K^{1-\frac{b_2}{2}}}
+
\frac{(\sigma^2 b^{-1} + G^2)\alpha}{2 \sqrt{v_*} \beta_{1} \hat{\beta}_1 K}
\sum_{k=1}^K \frac{1}{k^{\frac{b_2}{2}}}
+
D(\bm{\theta}^\star) G \frac{\hat{\beta}_{1}}{\beta_{1}}\\
&\quad + 
\hat{\beta}_{1}
D(\bm{\theta}^\star) \left(B + \sqrt{\frac{\sigma^2}{b} + G^2} \right).
\end{align*}
We also have that 
\begin{align}\label{int}
\begin{split}
\frac{1}{K} \sum_{k=1}^K \frac{1}{k^{\frac{b_2}{2}}}
&\leq
\frac{1}{K} \left( 
1 + \int_1^K \frac{\mathrm{d}t}{t^{\frac{b_2}{2}}}
\right)
= 
\frac{1}{K} \left\{ 
1 + \left[ \left( 1 - \frac{b_2}{2} \right) t^{ 1 - \frac{b_2}{2}} \right]_1^K
\right\}\\
&\leq
\frac{1}{K} \left\{
1 
+ \left( 1 - \frac{b_2}{2} \right)K^{ 1 - \frac{b_2}{2}}
\right\}
\leq
\frac{2}{K}K^{ 1 - \frac{b_2}{2}} 
= 
\frac{2}{K^{\frac{b_2}{2}}}.
\end{split}
\end{align}
Hence, 
\begin{align*}
\frac{1}{K}\sum_{k=1}^K \mathbb{E}\left[ (\bm{\theta}_k - \bm{\theta}^\star)^\top \bm{m}_{k} \right]
&\leq
\frac{d \tilde{D}(\bm{\theta}^\star) \sqrt{M}}{2 \beta_1 \alpha K^{1-\frac{b_2}{2}}}
+
\frac{(\sigma^2 b^{-1} + G^2)\alpha}{\sqrt{v_*} \beta_{1} \hat{\beta}_1
K^{\frac{b_2}{2}}}
+
D(\bm{\theta}^\star) G \frac{\hat{\beta}_{1}}{\beta_{1}}\\
&\quad + 
\hat{\beta}_{1}
D(\bm{\theta}^\star) \left(B + \sqrt{\frac{\sigma^2}{b} + G^2} \right).
\end{align*}
A discussion similar to the one showing the above inequality and Lemma \ref{lem:4_1} imply that 
\begin{align*}
\frac{1}{K} \sum_{k=1}^K
\mathbb{E}\left[
\nabla f(\bm{\theta}_k)^\top (\bm{\theta}_k - \bm{\theta}) 
\right]
&\leq
\frac{d \tilde{D}(\bm{\theta}) \sqrt{M}}{2 \alpha \beta_1 K^{1 - \frac{b_2}{2}}}
+
\frac{\alpha}{\sqrt{v_*} \beta_{1} (1-\beta_1) K^{\frac{b_2}{2}}}\left(\frac{\sigma^2}{b} + G^2 \right)\\ 
&\quad
+
\frac{1 - \beta_{1}}{\beta_{1}} D(\bm{\theta}) G
+
\left( \frac{1}{\beta_{1}} + 2 (1 - \beta_{1}) \right)
D(\bm{\theta})
\left( B + \sqrt{\frac{\sigma^2}{b} + G^2} \right), 
\end{align*}
which completes the proof.
\end{proof}

\subsection{Proof of Theorem \ref{d2}}
\begin{proof}
Let
\begin{align*}
\alpha_k = \frac{1}{k^a}, \text{ }
\beta_{1k} = \beta_1, \text{ }
\beta_{2k} = \beta_2, \text{ }
\tilde{\beta}_{1k} = 1 - \beta_1^{k+1} \leq 1,\text{ }
\tilde{\beta}_{2k} = 1 - \beta_2^{k+1} \leq 1, \text{ }
\hat{\beta}_{1} = 1 - \beta_1.
\end{align*}
We have that $\tilde{\beta}_{2k} = 1 - \beta_{2k}^{k+1} \geq 1 - \beta_2$. Lemma \ref{lem:3_1} ensures that
\begin{align*}
\frac{1}{K}\sum_{k=1}^K \mathbb{E}\left[ (\bm{\theta}_k - \bm{\theta}^\star)^\top \bm{m}_{k} \right]
&\leq
\frac{d \tilde{D}(\bm{\theta}^\star) \sqrt{M} \tilde{\beta}_{1K}}{2 \beta_1 \alpha_K \sqrt{\tilde{\beta}_{2K}}K}
+
\frac{(\sigma^2 b^{-1} + G^2)}{2 \sqrt{v_*} \beta_{1} \hat{\beta}_1 K}
\sum_{k=1}^K \alpha_k \sqrt{\tilde{\beta}_{2k}} 
+
D(\bm{\theta}^\star) G \frac{\hat{\beta}_{1}}{\beta_{1}}\\
&\quad + 
\hat{\beta}_{1}
D(\bm{\theta}^\star) \left(B + \sqrt{\frac{\sigma^2}{b} + G^2} \right)\\
&\leq
\frac{d \tilde{D}(\bm{\theta}^\star) \sqrt{M}}{2 \beta_1 \sqrt{1-\beta_2}K^{1-a}}
+
\frac{(\sigma^2 b^{-1} + G^2)}{2 \sqrt{v_*} \beta_{1} \hat{\beta}_1 K}
\sum_{k=1}^K \frac{1}{k^a} 
+
D(\bm{\theta}^\star) G \frac{\hat{\beta}_{1}}{\beta_{1}}\\
&\quad + 
\hat{\beta}_{1}
D(\bm{\theta}^\star) \left(B + \sqrt{\frac{\sigma^2}{b} + G^2} \right),
\end{align*}
which, together with \eqref{int}, implies that 
\begin{align*}
\frac{1}{K}\sum_{k=1}^K \mathbb{E}\left[ (\bm{\theta}_k - \bm{\theta}^\star)^\top \bm{m}_{k} \right]
&\leq
\frac{d \tilde{D}(\bm{\theta}^\star) \sqrt{M}}{2 \beta_1 \sqrt{1-\beta_2}K^{1-a}}
+
\frac{(\sigma^2 b^{-1} + G^2)}{2 \sqrt{v_*} \beta_{1} \hat{\beta}_1 K^a} 
+
D(\bm{\theta}^\star) G \frac{\hat{\beta}_{1}}{\beta_{1}}\\
&\quad + 
\hat{\beta}_{1}
D(\bm{\theta}^\star) \left(B + \sqrt{\frac{\sigma^2}{b} + G^2} \right).
\end{align*}
A discussion similar to the one showing the above inequality and Lemma \ref{lem:4_1} implies the second assertion in Theorem \ref{d2}.
\end{proof}

\subsection{Proof of Theorem \ref{d3}}
\begin{proof}
The proofs of Theorems \ref{c3} and \ref{d2} lead to Theorem \ref{d3}.
\end{proof}

\vskip 0.2in

\end{document}